\documentclass{article}

\usepackage{amsmath,amsfonts,bm}

\def\eqref#1{equation~\ref{#1}}

\def\1{\bm{1}}

\def\eps{{\epsilon}}

\def\ve{{\bm{e}}}
\def\vf{{\bm{f}}}
\def\vg{{\bm{g}}}
\def\vh{{\bm{h}}}

\def\vk{{\bm{k}}}

\def\vm{{\bm{m}}}
\def\vn{{\bm{n}}}
\def\vo{{\bm{o}}}
\def\vp{{\bm{p}}}
\def\vq{{\bm{q}}}

\def\vs{{\bm{s}}}

\def\vu{{\bm{u}}}
\def\vv{{\bm{v}}}

\def\vx{{\bm{x}}}
\def\vy{{\bm{y}}}
\def\vz{{\bm{z}}}

\def\mK{{\bm{K}}}

\def\mQ{{\bm{Q}}}

\def\mV{{\bm{V}}}
\def\mW{{\bm{W}}}

\DeclareMathAlphabet{\mathsfit}{\encodingdefault}{\sfdefault}{m}{sl}
\SetMathAlphabet{\mathsfit}{bold}{\encodingdefault}{\sfdefault}{bx}{n}

\def\gI{{\mathcal{I}}}

\def\gS{{\mathcal{S}}}

\def\gX{{\mathcal{X}}}

\def\sF{{\mathbb{F}}}

\def\sI{{\mathbb{I}}}

\def\sR{{\mathbb{R}}}

\def\sZ{{\mathbb{Z}}}

\newcommand{\att}{\mathrm{Attn}}
\newcommand{\attl}{\mathrm{Attn}_{\text{linear}}}

\usepackage{microtype}
\usepackage{graphicx}
\usepackage{subfigure}
\usepackage{booktabs} 
\usepackage{wrapfig}

\usepackage{hyperref}

\usepackage[accepted]{icml2024}

\usepackage{amsmath}
\usepackage{amssymb}
\usepackage{mathtools}
\usepackage{amsthm}

\usepackage[capitalize,noabbrev]{cleveref}

\theoremstyle{plain}
\newtheorem{theorem}{Theorem}[section]
\newtheorem{proposition}[theorem]{Proposition}
\newtheorem{lemma}[theorem]{Lemma}

\theoremstyle{definition}
\newtheorem{definition}[theorem]{Definition}
\newtheorem{assumption}[theorem]{Assumption}
\theoremstyle{remark}

\usepackage[textsize=tiny]{todonotes}

\icmltitlerunning{Do Efficient Transformers Really Save Computation?}

\begin{document}

\twocolumn[
\icmltitle{Do Efficient Transformers Really Save Computation?}



\icmlsetsymbol{equal}{*}

\begin{icmlauthorlist}
\icmlauthor{Kai Yang}{pkucs}
\icmlauthor{Jan Ackermann}{eth}
\icmlauthor{Zhenyu He}{pkuai}
\icmlauthor{Guhao Feng}{pkucs}
\icmlauthor{Bohang Zhang}{pkuai}
\icmlauthor{Yunzhen Feng}{nyu}
\icmlauthor{Qiwei Ye}{zhiyuan}\\
\icmlauthor{Di He}{pkuai}
\icmlauthor{Liwei Wang}{pkuai,pkuml}

\end{icmlauthorlist}

\icmlaffiliation{pkuai}{National Key Laboratory of General Artificial Intelligence, School of Intelligence Science and Technology, Peking University}
\icmlaffiliation{pkucs}{School of EECS, Peking University}
\icmlaffiliation{pkuml}{Center for Machine Learning Research, Peking University}

\icmlaffiliation{eth}{ETH Zürich}
\icmlaffiliation{nyu}{New York University}
\icmlaffiliation{zhiyuan}{Beijing Academy of Artificial Intelligence}

\icmlcorrespondingauthor{Di He}{dihe@pku.edu.cn}
\icmlcorrespondingauthor{Liwei Wang}{wanglw@pku.edu.cn}
\icmlcorrespondingauthor{Bohang Zhang}{zhangbohang@pku.edu.cn}

\icmlkeywords{Machine Learning, Efficient Transformer, Chain-Of-Thought, ICML}

\vskip 0.3in
]



\printAffiliationsAndNotice{}  

\begin{abstract}
As transformer-based language models are trained on increasingly large datasets and with vast numbers of parameters, finding more efficient alternatives to the standard Transformer has become very valuable. While many efficient Transformers and Transformer alternatives have been proposed, none provide theoretical guarantees that they are a suitable replacement for the standard Transformer. This makes it challenging to identify when to use a specific model and what directions to prioritize for further investigation.
In this paper, we aim to understand the capabilities and limitations of efficient Transformers, specifically the Sparse Transformer and the Linear Transformer. We focus on their reasoning capability as exhibited by Chain-of-Thought (CoT) prompts and follow previous works to model them as Dynamic Programming (DP) problems.
Our results show that while these models are expressive enough to solve general DP tasks, contrary to expectations, they require a model size that scales with the problem size. 
Nonetheless, we identify a class of DP problems for which these models can be more efficient than the standard Transformer.
We confirm our theoretical results through experiments on representative DP tasks, adding to the understanding of efficient Transformers' practical strengths and weaknesses.
\end{abstract}
\vspace{-10pt}

\section{Introduction}
\begin{table*}[t]
    \vspace{-10pt}
    \small
    \centering
    \setlength{\tabcolsep}{4pt}
    \caption{Complexity of the Transformer variants on different tasks.}
    \begin{tabular}{c|ccc}
    \hline
    Architecture & General Reasoning & Arithmetic & Reasoning (locality assumption) \\ \hline
    Standard Transformer  & $\tilde\Theta(L^2)$ & $\tilde\Theta(L^2)$  & $\tilde\Theta(L^2)$ \\
    Sparse Transformer    & $\tilde\Theta(L^2)$ & $\tilde\Theta(L\sqrt L)$ & $\tilde\Theta(L\sqrt L)$ if $m=\mathrm{O}(\sqrt L)$ \\
    Linear Transformer    & $\tilde\Theta(L^2)$ & $\tilde\Omega(L\sqrt L)$ & $\tilde\Omega(mL)$ \\ \hline
    \end{tabular}
    \vspace{-0.4cm}
    \label{tab:complexity}
\end{table*}

The Transformer architecture, as introduced in the seminal work of \citet{vaswani2017attention}, has demonstrated a remarkable performance in numerous applications ranging from natural language processing to computer vision and speech. A significant advancement has recently been made, by scaling up Transformers to build Large Language Models (LLMs) \citep{brown2020language,openai2023gpt4,touvron2023llama2}. These LLMs, exemplified by models like GPT and LLaMa, typically have billions of parameters and are trained on datasets containing trillions of tokens. Given the substantial computational demands, enhancing LLMs' efficiency has become a pivotal research focus in academic and industrial contexts.

The primary computational bottleneck in Transformers arises from the self-attention module, whose complexity scales quadratically with the sequence length. The cost becomes particularly noticeable in tasks that require \emph{long sequence generation}, such as coherent story generation or reasoning with Chain-of-Thought prompts~\citep{wei2022chain,kojima2022large,nye2022show,zhou2023leasttomost}. Given the practical needs, a large body of work seeks to develop \emph{efficient Transformers} that can reduce the quadratic complexity of self-attention \citep{tay2022efficient}, typically by imposing sparsity into architectural design~\citep{child2019generating,beltagy2020Longformer,qiu2020blockwise,kitaev2020reformer,vyas2020fast,roy2021efficient} or by employing low-rank or kernel-based approximations to accelerate the computation~\citep{katharopoulos2020transformers,choromanski2021rethinking,peng2021random,wang2020linformer,luo2021stable}. However, there is generally a lack of understanding about the capabilities of efficient Transformer.

In this work, we take a step towards theoretically understanding the capability of efficient Transformers. In particular, we focus on the models' reasoning ability, a fundamental aspect of human intelligence that plays a vital role in problem-solving, decision-making, and planning. Inspired by a recent study in \citet{feng2023towards}, we model reasoning as a \emph{dynamic programming} (DP) process as it closely resembles the way Chain-of-Thought prompts are executed. The output sequence consists of answers to a series of intermediate steps, each corresponding to solving a subproblem represented by a DP state. \citet{feng2023towards} proved that all reasoning problems fitting within this framework can be solved by a standard autoregressive Transformer of a \emph{constant} size (irrelevant to the problem scale), thus achieving a computational complexity of $\Theta(L^2)$ where $L$ is the length of the output sequence.

In our work, we focus on two representative~\citep{tay2022efficient} and successful~\citep{tay2020long,brown2020language} variants of efficient Transformers: the Sparse Transformer \citep{child2019generating} and the Linear Transformer \citep{katharopoulos2020transformers}. (In the following we will refer to these two as efficient Transformers.)
Our analysis shows that both architectures possess the necessary expressiveness for all problems within this DP framework despite only scaling with $\Theta(L\sqrt L)$ and $\Theta(L)$.
Although this positive result might lead one to believe that we can supplant the standard Transformer with others of lower complexity,
the situation is more complicated: our main result highlights that both Sparse Transformer and Linear Transformer require a \emph{growing} model size with respect to the problem scale $L$, in contrast to the \emph{constant} size of standard Transformers. Specifically, under mild assumptions, we prove that neither architecture can generate the DP solution unless the hidden dimension of the network layers scales as $\tilde\Omega(\sqrt L)$. This scaling results in a total computational complexity of $\tilde\Omega(L^2)$, matching the vanilla Transformer's complexity. But this result introduces a paradox: when tackling general DP problems, the touted efficiency of these ``efficient'' Transformers appears to dissolve, rendering them comparably efficient to standard Transformers.

The above findings about general DP problems raise the question: For which problems are efficient Transformers efficient? To answer this question, we start by studying a fundamental task for reasoning: evaluating arithmetic expressions~\citep{feng2023towards}. Notably, we find that the complexity lower bound can be improved to $\tilde\Omega(L\sqrt L)$ for both architectures, and the lower bound can be attained for the Sparse Transformer with a constant hidden dimension. Motivated by this finding, we then identify a general condition that unlocks the efficiency of efficient Transformers, called the \emph{locality} assumption. Intuitively, this assumption states that each step in the reasoning process only depends on the outcome of \emph{recent} $m$ reasoning steps where $m$ is far smaller than $L$, i.e. $m=o(L)$. Under this assumption, we show that the complexity lower bound can be improved for sparse- and Linear Transformer. We summarize our theoretical results in \cref{tab:complexity}.

We complement our theoretical findings with an extensive set of experiments. Following~\citet{feng2023towards}, we focus on the Arithmetic task and two additional DP problems: the Longest Increasing Subsequence (LIS) and the Edit Distance (ED). Notably, the ED task satisfies the locality assumption, whereas the LIS task does not. For each task, we systematically investigate how variations in the problem size (i.e., the sequence length $L$) and the hidden dimension of the Transformer models impact the models' performance. Empirical evidence confirms that, for both efficient Transformers, the required hidden dimension increases as the problem size grows in most scenarios, while this is not the case for the standard Transformer. Moreover, the dependency between hidden dimension and problem scale is more pronounced in LIS than in ED. These results validate our theory and offer practical insights into the strengths and weaknesses of efficient Transformers.

\textbf{Notations}. We adopt the big-O notation throughout this paper. Specifically, given two functions $f, g : \gX \to [0, \infty)$ where $\gX$ can be any set, we write $f = \mathrm{O}(g)$ if there exists a constant $c > 0$ such that $f(x) \le c g(x)$ for all $x \in \gX$. We also write $f = \Omega(g)$ if $g = \mathrm{O}(f)$, and write $f = \Theta(g)$ if both $f = \mathrm{O}(g)$ and $f = \Omega(g)$ hold. Moreover, given two functions $f, g : \mathbb N_+^d \to [0, \infty)$, we write $f = \tilde{\mathrm{O}}(g)$ if there exist constants $c,k > 0$ such that $f(x) \le c g(x)\log^{k} (x_1\cdots x_d)$ for all $x \in \mathbb N_+^d$. The notations $\tilde\Omega(\cdot)$ and $\tilde\Theta(\cdot)$ can be similarly defined.

\section{Related Work}
Transformers and Large Language Models have received significant attention due to their unprecedented success across various domains. A considerable body of literature has emerged to establish a deeper theoretical understanding of their strengths and constraints.

\textbf{Universal Approximation.} Initially, the theoretical focus was on the capacity of Transformers to approximate diverse functions. \citet{yun2019transformers}  postulated that adequately sized Transformers can universally approximate any continuous sequence-to-sequence functions within certain bounds. A parallel line of work first showed that Transformers with infinite precision are turing-complete~\citep{perez2019turing, perez2021attention} and later \citet{wei2022statistically} established that Transformers with finite precision are approximately turing-complete. Recently, \citet{alberti2023sumformer} proved that Linear Transformers are also universal approximators. Whereas these results approach expressiveness by proving computational capacity, we complement our expressiveness results with complexity lower bounds for practical settings.

\textbf{Formal Language Learning.}
Additionally, the Transformer's expressivity has been studied in the context of formal language learning. \citet{bhattamishra2020ability} constructed a Transformer that detects counter languages, and \citet{yao2021self} show how to detect Dyck languages. \citet{liu2022transformers} show that shallow Transformers can learn finite state automata and simulate them for a number of steps that scale with the model size. Conversely, \citet{hahn2020theoretical} shows that transformers can not learn distributions over languages. 
Other works use classical techniques from circuit complexity~\citep{furst1984parity} to prove that Transformers can simulate classes of circuits~\citep{hao2022formal, merrill2022saturated, merrill2023parallelism}.

\textbf{Measuring Complexity.} \citet{weiss2021thinking} introduce a programming language that maps to learnable Transformer encoders and facilitates the analysis of the complexity of problems with respect to layers and attention heads. \citet{sanford2023representational} introduce a sparse averaging task that requires recurrent and feed-forward networks to be of linear complexity, whereas the Transformer only needs to scale logarithmically. These works are similar to ours in that we establish concrete relationships between model complexity and solvability of the posed problems. But our work deals with autoregressive efficient Transformers equipped with Chain-of-Thought.

\textbf{In-context learning.} A recent approach shows its in-context learning ability~\citep{garg2022what, brown2020language}. Following this, there are also theoretical results that~\citep{dai2023can,von2023transformers,akyurek2022learning} prove it can perform gradient descent. Another line of work shows in-context-learning via induction heads~\citep{elhage2021mathematical, olsson2022context}. Similarly, \citet{feng2023towards} show that auto regressive transformers can learn to perform dynamic programming when equipped with Chain-of-Thought. While in the same setting as Feng et al., we investigate efficient Transformers and present a problem class that encourages efficiency.

\textbf{Efficient Transformer.} Due to the high complexity of the attention layer, many more efficient methods have been proposed. A first series of ideas exploit fixed attention patterns~\citep{child2019generating, beltagy2020Longformer, qiu2020blockwise}. Another line of work approximates the attention as a low rank matrix or with kernels~\citep{katharopoulos2020transformers, wang2020linformer, choromanski2021rethinking} and further works deal with learned patterns~\citep{kitaev2020reformer, tay2020sparse, roy2021efficient}. A last set of works even completely move away from transformers~\citep{sun2307retentive, gu2023mamba}.
Two recent works study when standard attention can be efficient~\citep{alman2023fast} and how to approximate standard attention in linear time~\citep{keles2023computational}.
In contrast to their work, we give theoretical analyses for existing and popular efficient Transformers.

\section{Efficient Transformers}
\label{sec:preliminary}

The autoregressive Transformer, also called the decoder-only Transformer \citep{radford2019language,dai2019transformer}, is a sequence-to-sequence neural network defined as follows. Given an input sequence $\vs$ of length $n$, it first transforms each input token $s_i$ ($i\in[n]$) into a $D$-dimensional vector $\vx^{(0)} = \text{Embed}(s_i)+\vp_i\in \mathbb{R}^D$, where $\text{Embed}(\cdot)$ is the token embedding layer and $\vp_i$ a learnable positional embedding. Then, $M$ Transformer blocks follow, the $l$-th of which has the following form: 
\begin{align}
\label{eq:transformer_1}
    \vh^{(l)}_i&=\vx^{(l-1)}_i+\att^{(l)}(\vx^{(l-1)}_i;\{\vx^{(l-1)}_j:j\in [i]\}),\\
\label{eq:transformer_2}
    \vx^{(l)}_i&=\vh^{(l)}_i+\mathrm{FFN}^{(l)}(\vh^{(l)}_i)
\end{align}
Here, $\att^{(l)}$ and $\mathrm{FFN}^{(l)}$ denote the multi-head self-attention layer and the feed-forward network of the $l$-th Transformer block, respectively:
\begin{align}
\label{eq:attention}
   &\att^{(l)}(\vx, \gS)=\sum_{h=1}^H\left(\mW_\text{O}^{(l,h)}\right)^\top \cdot \mathrm{H}^{(l, h)}(\vx, \gS),\\
   &\mathrm{H}^{(l, h)}(\vx, \gS)\! =\!\frac{\sum_{\vz\in\gS}\exp{\left((\mW_\text{K}^{(l,h)}\vz)^\top(\mW_\text{Q}^{(l,h)}\vx)\right)} \mW_\text{V}^{(l,h)}\vz}{\sum_{\vz\in \gS}\exp{\left((\mW_\text{K}^{(l,h)}\vz)^\top(\mW_\text{Q}^{(l,h)}\vx)\right)}},\\
\label{eq:ffn}
    &\mathrm{FFN}^{(l)}(\vx)=\mW_2^{(l)}\sigma(\mW_1^{(l)}\vx),
\end{align}
where $\mW_\text{Q}^{(l,h)},\mW_\text{K}^{(l,h)},\mW_\text{V}^{(l,h)},\mW_\text{O}^{(l,h)}\in \mathbb R^{\lceil \frac D H\rceil\times D}$ are the query, key, value, output matrices of the $h$-th head in the $l$-th layer, respectively, and $\mW_1^{(l)},\mW_2^{(l)}\in \mathbb R^{D\times D}$ are weight matrices in the FFN. The activation $\sigma$ is chosen as GeLU \citep{hendrycks2016gaussian}, following \cite{radford2019language,devlin2019bert}.
The computed embedding $\vx_n^{(M)}$ will be used to predict the next token $s_{n+1}$, which is then concatenated to the input to continue the sequence generation process. The process stops when an \texttt{End-of-Sentence} token is generated.

Based on \cref{eq:transformer_1,eq:transformer_2,eq:attention,eq:ffn}, it is easy to see that the computational complexity of an autoregressive Transformer is $\Theta(M(L^2 D + L D^2))$, where $L$ is the sequence length. This quadratic dependency on $L$ limits the application of Transformers to long text, in particular for complex reasoning tasks. To battle this, researchers have proposed various efficient Transformers to reduce the complexity. In our work, we investigate the Sparse Transformer and the Linear Transformer. Below, we describe the two architectures which are studied in this paper.

\textbf{Sparse Transformer}. Unlike the standard Transformer where each token $\vx^{(l)}$ can attend to all previous positions $\{\vx^{(l)}_j:j\in [i]\}$ (see \cref{eq:transformer_1}), in a Sparse Transformer it only attends to a \emph{subset} of previous tokens $\{\vx^{(l)}_j:j\in \gI_i\}$. In this paper, we study a standard design paradigm proposed in \citet{child2019generating}, which employs a block-wise pattern as shown in the following:
\begin{align}
\label{eq:sparseattention}
    \gI_i=\{j:i-kB< j\le i\}\cup\{j:j-1\bmod B\ge B-c\}
\end{align}
where $B$ is called the block size and $k,c$ are constant integers. When $B=\Theta(\sqrt L)$, the Sparse Transformer achieves a minimal complexity of $\Theta(M(L\sqrt L D + L D^2))$. We note that GPT-3 adopted the above design paradigm \citep{brown2020language}.

\textbf{Linear Transformer}. Another line of work proposed to accelerate the attention computation (\cref{eq:attention}) using kernel-based approximations. A representative approach is the Linear Transformer \citep{katharopoulos2020transformers}, which approximates  $\mathrm{Attn}^{(l)}$ with the following formula:
\begin{align}
\label{eq:linearattention}
    &\attl^{(l)}(\vx, \gS)=\sum_{h=1}^H\left(\mW_\text{O}^{(l,h)}\right)^\top \cdot  H_{\text{linear}}^{(l, h)}(\vx, \gS),\\
     &H_{\text{linear}}^{(l, h)}(\vx, \gS) = \frac{\sum_{\vz\in \gS}\phi(\mW_\text{K}^{(l,h)}\vz)^\top\phi(\mW_\text{Q}^{(l,h)}\vx)(\mW_\text{V}^{(l,h)}\vz) }{\sum_{\vz\in \gS}\phi(\mW_\text{K}^{(l,h)}\vz)^\top\phi(\mW_\text{Q}^{(l,h)}\vx)}
\end{align}
where they choose $\phi(\vx)=\mathrm{elu}(\vx)+\mathbf 1$. The above computation can be accelerated by rearranging the order of computation so that the intermediate results $\sum_{\vz\in \gS}(\mW_\text{V}^{(l,h)}\vz)\phi(\mW_\text{K}^{(l,h)}\vz)^\top$ and $\sum_{\vz\in \gS}\phi(\mW_\text{K}^{(l,h)}\vz)^\top$ associated with different $\gS$ can be jointly computed using prefix sum, finally yielding a complexity of $\Theta(ML D^2)$ which is linear in $L$.

\section{Expressiveness of Efficient Transformers in Reasoning Tasks}
\label{sec:DP}
Reasoning constitutes a fundamental aspect of human intelligence and plays a vital role in problem-solving, decision-making, and planning. Recently, Transformer-based LLMs have demonstrated remarkable reasoning abilities \citep{openai2023gpt4,touvron2023llama2}. This has sparked a series of studies aimed at theoretically understanding how powerful these models are. In particular, \citet{feng2023towards} recently revealed that autoregressive Transformers are capable of solving a general class of reasoning problems formalized as Dynamic Programming (DP).  In this section, we extend this finding by investigating how things change when moving to various types of efficient Transformers.

\subsection{Problem formulation}

Dynamic programming decomposes a complex reasoning problem into a sequence of reasoning steps, each of which corresponds to a subproblem and is called a DP state. Different subproblems depend on each other because they can be efficiently solved based on the answers of previously solved subproblems. Formally, denoting by $\mathsf{dp}(i)$ the answer of subproblem $i$, then the relation between subproblems can be characterized using a transition function:
\begin{equation}
\label{eq:dp_transition}
    \mathsf{dp}(i)=f\left(i, \mathsf{dp}(h_1(i)),\cdots,\mathsf{dp}(h_K(i)), s_{g_1(i)},\cdots,s_{g_J(i)}\right)
\end{equation}
where $\vs$ is the input sequence, and $f$, $g_1,\cdots,g_J$, $h_1,\cdots,h_K$ are functions that depends on the problem. In other words, the answer of each subproblem is fully determined by the answers of a finite number of previous subproblems plus a finite number of input tokens. Based on \cref{eq:dp_transition}, we can sequentially solve all subproblems one by one. After solving all subproblems, the final answer can be computed by $u(\mathsf{dp}(i_N))$, where $i_N$ is the last DP state and $u$ is a problem-dependent function. By defining our problem so generally, we also cover CoT problems. 
We assume that the $f$, $\vg$, $\vh$ and $u$ above can be approximated by an MLP with GeLU activation of constant size. We also assume that during the CoT generation process, the next state can be obtained by an MLP where the input is the current state. One can refer to Appendix \ref{appendix:proof_secDP} for a formal description. We argue that these assumptions are mild and that they have been used in previous work~\cite{feng2023towards}.

In our subsequent analysis, without loss of generality, we assume that each input element $s_j$ is an integer, and each state $i$, DP value $\mathsf{dp}(i)$, and the final answer can all be represented by vectors of integer elements. The domain of these integers can grow polynomially with respect to the length $L$.

\textbf{Output format}. Following \citet{feng2023towards}, given a DP task and an input seuqence $\vs$, an autoregressive Transformer generates the answer with all intermediate steps in the following form:
\begin{align}
\label{eq:dp_format}
&(s_1,\mathbf 0,\mathbf 0,\mathbf 0)\!\quad\! \ldots\quad (s_n,\mathbf 0,\mathbf 0,\mathbf 0)\!\quad\! | \nonumber\\
&(0,i_1, \mathsf{dp}(i_1),\mathbf 0) \!\quad\! \ldots \!\quad\! (0,i_N, \mathsf{dp}(i_N),\mathbf 0)\nonumber\\  
&(0,\mathbf 0,\mathbf 0, u(\mathsf{dp}(i_N)))
\end{align}
Here, the subsequence ending at the special token ``$|$'' is the input to the Transformer, and the remainder will be autoregressively generated. The output at each position is split into four parts that store the input, state, DP value, and final answer, respectively. We denote by $i_1,\cdots,i_N$ the sequence of DP states representing all subproblems in order. We consider the regression setting where the output at each position is simply obtained from the embedding of the last Transformer layer by projecting each dimension to the nearest integer. Similarly, each generated output directly serves as the input of the next position (without using a token embedding layer).

\textbf{Log-precision Transformers}. We adopt a realistic and widely-used setting where all internal neurons in the Transformer can only store floating-point numbers within a finite $O(\log L)$ bit precision \citep{merrill2023parallelism,liu2023transformers,feng2023towards}, and all basic floating-point computations are truncated, as implemented on a computer. Log-precision implies that each neuron has a limited capacity for computation and information storage. Nevertheless, they remain powerful as they can represent a large range of values (i.e., polynomial in the sequence length $L$), recovering important quantities like positional embedding.

Under the above assumptions, \citet{feng2023towards} proved the following main result for the standard transformer:
\begin{theorem}[informal]
\label{thm:baseline}
    Consider any DP problem defined above that satisfies the assumptions from \ref{appendix:proof_secDP}. For any integer $n>0$, there exists a log-precision autoregressive Transformer with a constant depth $M$, a constant hidden dimension $D$, and a constant number of attention heads $H$ (independent of $n$) that can generate the correct output for all inputs $\vs$ of length $n$.
\end{theorem}

\subsection{Main results}

Now, we investigate whether the efficient Transformers defined in \cref{sec:preliminary} are as powerful as the standard Transformer in solving DP problems. In particular, we establish a similar result to \cref{thm:baseline}, which we present in the theorem below:

\begin{theorem}
\label{thm:general_upper_bound}
    Consider any DP problem satisfying the same condition as in \cref{thm:baseline}. Given any integer $n>0$, let $L$ be the length of the output sequence when the input sequence length is $n$. Then, for both (log-precision) Sparse Transformer with block size $B=\Theta(\sqrt L)$ and Linear Transformer, there is a model with a constant depth $M$, a constant number of attention heads $H$, and a hidden dimension $D=\mathrm{O}(\sqrt L)$ that can generate the correct output for all inputs $\vs$ of length $n$.
\end{theorem}

The proof of \cref{thm:general_upper_bound} is non-trivial and is deferred to \cref{appendix:upper_dp}. In the proof, we give explicit constructions of parameters for sparse/linear attention and FFN layers, showing that these layers can implement a set of basic operations presented in \cref{sec:lemmas}. We then use these operations as building blocks to form a complete model that solves the DP task.

\cref{thm:general_upper_bound} suggests that replacing the standard self-attention with these efficient variants does not restrict the model's expressiveness in reasoning.
However, when we compare the total complexity $\mathrm{O}(L^2)$ of the derived models, we can see that it is the same as for the standard attention, which only needed $d=\mathrm{O}(1)$.

\textbf{Is the increase in model size necessary?} \cref{thm:general_upper_bound} only gives a complexity \emph{upper bound} for Sparse/Linear Transformers. It remains to show whether the bound is \emph{tight} and what the \emph{lower bound} of the required hidden dimension is. To answer this question, we will focus on a restricted class of DP problems which we call \emph{regular} DP problems:

\begin{definition}
\label{def:regular}
    A DP problem is called \emph{regular} if for any two different input sequences $\vs^{(1)}$ and $\vs^{(2)}$ (of the same length) and a fixed but arbitrary model that solves the DP problem, there is a state $i$ such that $\mathsf{dp}(i)$ is different between input $\vs^{(1)}$ and $\vs^{(2)}$.
\end{definition}

We remark that regularity is a weak assumption, which only states that the reasoning process (not the final answer) should not be exactly the same when the input changes. For example, it excludes the case where the whole DP process does not depend on a specific input element $s_j$. Equipped with the regularity assumption, we present a central impossibility result:

\begin{theorem}
\label{thm:general_lower_bound}
    Consider any regular DP problem satisfying the same condition as in \cref{thm:baseline}. Assume that the output sequence length $L$ is proportional to the input sequence length $n$, i.e., $L=\Theta(n)$. Then, given a sufficiently large $n$, for both (log-precision) Sparse Transformer with block size $B=\Theta(\sqrt L)$ and Linear Transformer, a model with a constant depth $M$ and a constant number of attention heads $H$ can generate the correct output for all inputs $\vs$ of length $n$ only if the hidden dimension $D=\tilde\Omega(\sqrt L)$.
\end{theorem}

As presented in \cref{appendix:lower_dp}, the proof of \cref{thm:general_lower_bound} is based on the following finding: there are inherent information \emph{bottlenecks} in both types of efficient Transformers. Here, the bottleneck is a set of neurons whose values completely determine all ensuing outputs from a specific position. Due to the log-precision assumption, these neurons only store a limited amount of information. Hence it is only possible to recover all subsequent outputs when the hidden dimension is $\tilde\Omega(\sqrt L)$ --- otherwise, the Pigeonhole principle will imply that there are two different input sequences that share the same set of neuron values, yielding the contradiction by \cref{def:regular}.

\section{When Can Efficient Transformers Really Save Computation?}
\label{sec:arith}
In the previous section, we showed the surprising result that these efficient Transformers may not lead to reduced complexity compared to the standard Transformers in general reasoning tasks. However, one should not hastily jump to the conclusion that these efficient Transformers are always inefficient. In this section, we will discuss in which situations efficient Transformers are efficient.

\subsection{A motivating example: evaluating arithmetic expressions}

We begin by investigating a less complex task proposed by \citet{feng2023towards}, called the arithmetic evaluation. The task is to evaluate an arithmetic expression like ``$2\times (1+5)\div 4$'' and the complete output sequence looks like ``$2\times (1+5)\div 4=2\times 6\div 4=12\div 4=3$''. \citet{feng2023towards} proved that a standard Transformer of a constant size can solve this task. Surprisingly, we find that a similar result also holds for a constant-size Sparse Transformer, as shown in the proposition below:

\begin{proposition}[informal]
\label{thm:sparse_arithmetic}
    For any integer $n$, there exists a log-precision Sparse Transformer with block size $B=\Theta(\sqrt{L})$, 5 layers, 5 attention heads per layer, a constant hidden dimension that can generate the correct output for the arithmetic evaluation task for all expressions of length no more than $n$.
\end{proposition}

Owing to the constant dimension, the complexity of the arithmetic evaluation task can be reduced to $\mathrm{O}(L\sqrt L)$ by using Sparse Transformers. We next turn to Linear Transformers. While we do not give explicit constructions of model parameters that can solve the arithmetic task, one can still derive a lower bound by using a similar analysis as in \cref{thm:general_lower_bound}:

\begin{proposition}[informal]
\label{thm:linear_arithmetic}
    For any integer $n$, a log-precision Linear Transformer with a constant depth $M$ and a constant number of heads $H$ can generate the correct output for the arithmetic evaluation task for all expressions of length no more than $n$ only if the hidden dimension $D=\tilde\Omega(\sqrt[4] L)$.
\end{proposition}

Based on this result, the complexity lower bound of Linear Transformers scales like $\tilde\Omega(L\sqrt L)$, which, interestingly, matches that of Sparse Transformers and is also strictly less than $\tilde\Omega(L^2)$. This finding naturally raises the following question: Why do these efficient Transformers no longer require a large hidden dimension on the arithmetic task?
We will answer this question in the next subsection.

\subsection{Locality encourages efficiency}

A key difference of the arithmetic task compared to general DP is that its reasoning process exhibits inherent \emph{structures}. To be specific, the output sequence of arithmetic computation can be partitioned into blocks (separated by the symbol ``=''), where the content of each block depends solely on the preceding one and is irrelevant to other historical blocks. This paradigm is often named as data locality in computer science literature and is also common in general reasoning processes that follow the so-called Chain of Thought format \citep{wei2022chain}. In light of this, we consider a special class of DP problems dubbed the $m$-locality DP, which is formally defined below:
\begin{definition}[$m$-locality DP]
    Consider a DP problem with output sequence $\vo_1,\cdots,\vo_L$ of the form (\ref{eq:dp_format}) where $\vo_1,\cdots,\vo_n$ is the input sequence. The DP problem is said to satisfy the $m$-locality condition for some $m\ge n$, if there exist functions $f,h_1,\cdots,h_K$ such that for all $i\in[L]$, $\vo_i=f(\vo_{h_1(i)},\cdots,\vo_{h_K(i)})$, where $i-m\le h_k(i)<i$ for $k\in[K]$.
\end{definition}

In other words, the $m$-locality condition simply says that each DP state only depends on recent $m$ DP states. Note that the assumption $m\ge n$ is necessary to ensure that all inputs contribute to the answer of the DP problem. Below, we will discuss how the required hidden dimension of efficient Transformers can be reduced when $m$ is far smaller than $L$.

We first consider the Sparse Transformer, where we have the following result:
\begin{proposition}
\label{thm:sparse_locality}
    Consider any $m$-locality DP problem satisfying the same condition as in \cref{thm:baseline}. Given any integer $n>0$, let $L$ be the length of the output sequence when the input sequence length is $n$. Then, there exists a (log-precision) Sparse Transformer with block size $B=\Theta(m)$, a constant depth $M$, a constant number of attention heads $H$, and a constant hidden dimension $D$ that can generate the correct output for all inputs $\vs$ of length $n$.
\end{proposition}
As a result, the complexity of Sparse Transformer scales like $\tilde{\mathrm{O}}(mL)$, which is strictly less than $\tilde\Theta(L^2)$ when $m$ is far smaller than $L$. We next turn to the Linear Transformer, where we have the following lower bound:
\begin{proposition}
\label{thm:linear_locality}
    Consider any $m$-locality regular DP problem satisfying the same condition as in \cref{thm:baseline} and assume that $m=\Theta(n)$ where $n$ is the input sequence length. Then, a log-precision Linear Transformer with a constant depth $M$ and a constant number of heads $H$ can generate the correct output for all inputs $\vs$ of length $n$ only if the hidden dimension $D=\tilde\Omega(\sqrt m)$.
\end{proposition}

The above result implies that the complexity lower bound of Linear Transformer, which is imposed by the bottleneck, scales like $\tilde\Omega(mL)$, which is strictly less than $\tilde\Theta(L^2)$ when $m$ is far smaller than $L$. However, we remark that it remains a challenging open question of whether such a complexity lower bound can be matched.

\section{Experiments}

\begin{figure*}[!h]
    \centering
    \includegraphics[width=1.0\linewidth]{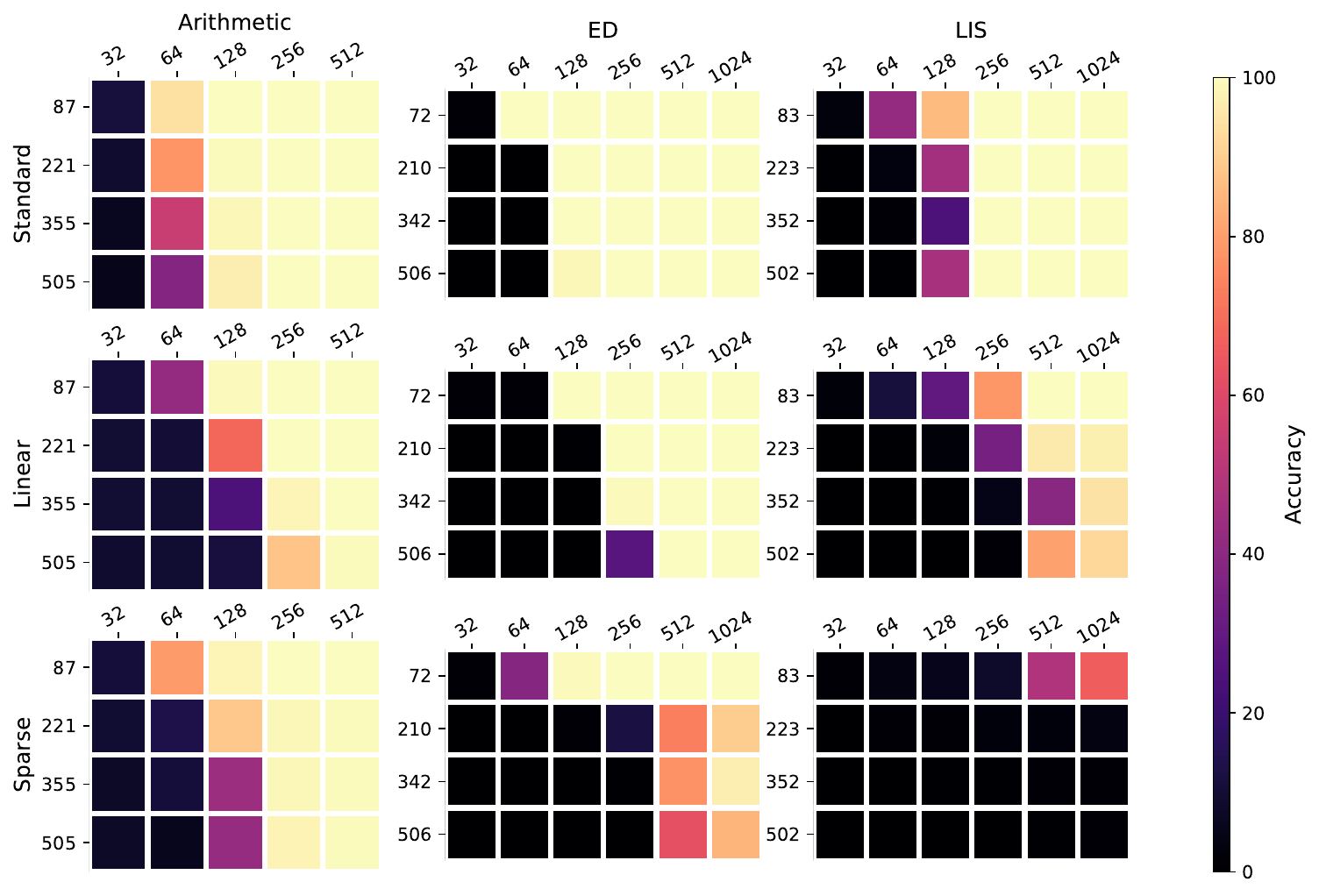}
    \vspace{-1cm}
    \caption{A comparison of accuracies across different tasks and model types. Each column corresponds to a task (Arithmetic, ED, LIS), and each row to a model (Standard Transformer, Linear Transformer, Sparse Transformer). Within each subplot, the x-axis represents the embedding dimension, and the y-axis denotes the problem size. The color intensities indicate the accuracy level achieved by the respective models. The figure demonstrates that efficient Transformers need larger hidden dimensions and that this requirement increases with problem size. It also highlights how standard Transformers can handle tasks across all difficulty levels with fixed embedding dimensions.}
    \label{fig:results}
    \vspace{-0.4cm}
\end{figure*}

In the preceding sections, we conducted a theoretical analysis to assess the capabilities of efficient Transformers for general DP problems and problems with locality. This section serves to validate those findings through comprehensive empirical experimentation. Inspired by the reasoning evaluation in \citep{feng2023towards}, we adopt a similar experimental design using common DP problems with Chain-of-Thought demonstrations. We focus on understanding how two key factors, problem size and the embedding dimension of the Transformer model, affect performance across tasks.

\subsection{Experimental Design}
\textbf{Tasks and datasets.} We chose three well-known tasks – LIS, ED, and Arithmetic – to represent a variety of problems in the DP domain. For the LIS task, the goal is to find the length of the longest increasing subsequence of a given integer sequence. The ED task's goal is to calculate the minimum cost required to convert one sequence to another using three basic edit operations: insert, delete, and replace. For the Arithmetic task, the goal is to calculate the correct result of an arithmetic expression consisting of numbers, addition, subtraction, multiplication, division, and brackets. The LIS task is the most general DP problem without locality property, while ED and Arithmetic exhibit higher locality.

Following previous work~\citep{feng2023towards}, we curate five datasets for each task, with different problem sizes and increasing difficulty. For the LIS task, the datasets encompass sequences of lengths $\{40,110,175,250\}$, equating to CoT lengths of $\{83,223,353,503\}$. For the ED task, the datasets span varying sequences of lengths – specifically, of averaged lengths 6, 12, 16, and 20, equating to maximum CoT lengths of $\{72,210,342,506\}$. For the Arithmetic task, the dataset consists of sequences with operator numbers in $\{6,10,13,16\}$, equating to maximum CoT lengths of $\{87,221,355,505\}$. Each training dataset has 1M samples, and each corresponding testing dataset has 0.1M.

\textbf{Model configurations.} For the standard Transformer model, we use the same configurations as used by~\citep{feng2023towards} {with 3 layers and 4 attention heads}, albeit with varying embedding dimensions. We employ sinusoidal positional embeddings and apply Xavier initialization to all parameters. For the activation function, we chose the standard GeLU, and the embedding dimensions for the ED and LIS tasks span the range ${32,64,128,256,512,1024}$ uniformly. As for the Arithmetic task, we exclude the $1024$ embedding dimensions as all the models have already performed well in $512$. The FFN layer's hidden dimension is four times the embedding dimension. We use the same configurations for both the Linear and Sparse Transformers. Within the Sparse Transformer, the block size $B$ is $2^{\lfloor \log_2(\sqrt{L}) \rfloor}$, with $L$ representing the upper limit of CoT length. Every experiment has the global token count $c$ fixed at 1 for every block. {Furthermore, we conduct the experiments on Transformer models with 5 layers while keeping other parameters the same. The additional results are shown in Appendix \ref{app:experiment}.}

\textbf{Model training and inference.}
In all experiments, we employ the AdamW optimizer~\citep{loshchilov2017fixing} with the following hyperparameters: $\beta _1 = 0.9, \beta _2 = 0.999, lr = 10^{-4}$, and $\text{weight decay} = 0.01$.
To enhance model generalization, we maintain a consistent dropout rate of 0.1. Our optimization minimizes the negative log-likelihood loss for all tokens in both the CoT steps and the answers. Each model does 100 training epochs with a batch size of 512. During the inference stage, the model generates the entire CoT process token by token, using greedy search until reaching the \texttt{End-of-Sentence} token. We evaluate the models' performance using the accuracy of the final answer, which is the last output in the sequence. We run all experiments on four V100 GPUs.

\subsection{Experimental Results}

\begin{table}[t]
    \caption{Minimum GFLOPs of different model types to achieve an accuracy above $90\%$ on Arithmetic Task.}
    \centering
    \resizebox{0.95\columnwidth}{!}{
    \begin{tabular}{ccccc}
    \toprule
    \textbf{Method} & \textbf{Length 87} & \textbf{Length 221} & \textbf{Length 355} & \textbf{Length 505} \\
    \midrule
    Standard & 0.027 & 0.266 & 0.426 & 0.605\\
    Linear & 0.422 & 4.220 & 6.756 & 9.584\\
    Sparse & 0.106 & 1.055 & 1.690 & 2.399 \\
    \bottomrule
    \end{tabular}
    }
    \vspace{-1em}
    \label{tab:2_GFLOPs}
\end{table}

Figure~\ref{fig:results} shows our main results. Each column of the figure corresponds to a particular task, and each row to a different model. Within each subplot, the x-axis shows the embedding dimension, and the y-axis indicates the problem size. The color intensities indicate the corresponding accuracies.

For almost all tasks and varying problem sizes, models with sufficiently large embedding dimensions can achieve nearly 100\% accuracy, except for the LIS task with the Sparse Transformer. Nevertheless, for this task, the accuracy still increases as the embedding dimension increases. This finding shows that efficient Transformers can handle these DP tasks with adequate expressiveness.

When we compare the subplots column-wise, it is evident that efficient Transformers generally need larger hidden dimensions than standard Transformers. Moreover, within each subplot of efficient Transformers, the required embedding dimension increases as the problem grows. In contrast, standard Transformers, with fixed embedding dimensions of 128 or 256, can handle tasks across all difficulty levels. These observations confirm our theoretical findings, suggesting that efficient Transformers are less efficient than previously perceived.
In \cref{tab:2_GFLOPs}, we further compare the minimal number of FLOPs required to achieve 90\% accuracy on Arithmetic for all models. The standard Transformer requires the lowest number of flops across all lengths.

Comparing the subplots row-wise, the growth in required embedding dimension with the model size becomes more pronounced as the locality decreases. This suggests that efficient Transformers are more efficient for DP tasks with strong locality, which aligns with our previous results.

\begin{figure}
    \centering
    \small
    \includegraphics[width=0.42\textwidth]{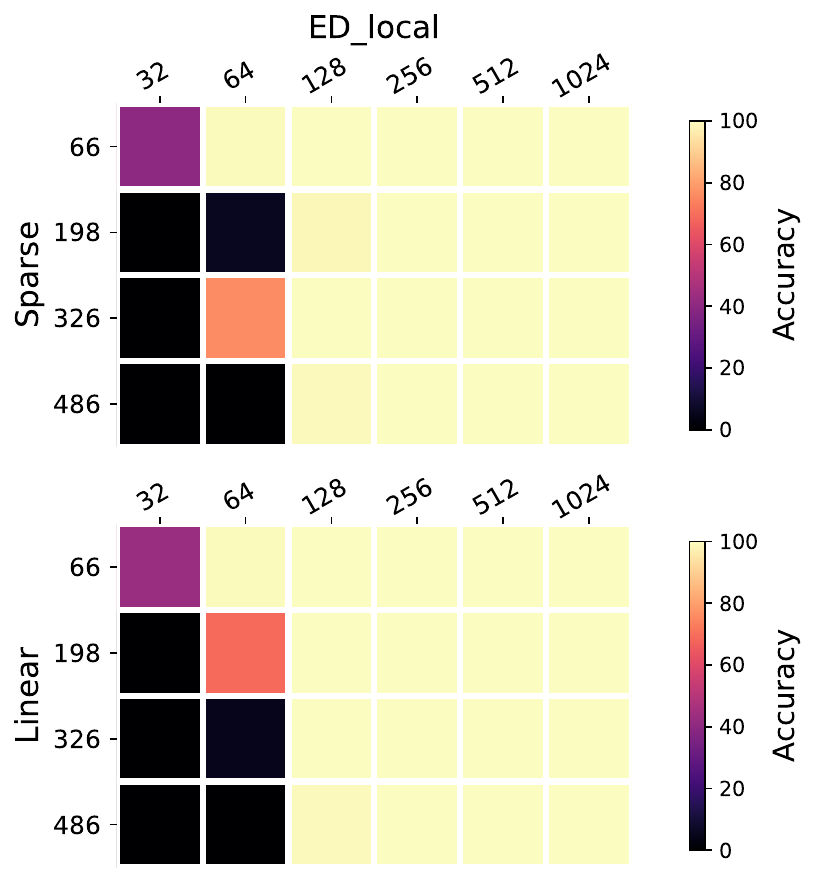}\\
    \vspace{-0.4cm}
    \caption{Accuracies of the Sparse Transformer and Linear Transformer on the ED\_Local task with varying problem size on the y- and embedding dimension on the x-axis. We can observe that both models benefit from locality.}
    \label{fig:locality}
    \vspace{-0.6cm}
\end{figure}

\textbf{Locality Study.} Although the previous results already indicate that problems with higher locality require higher embedding dimensions, we conducted another more explicit experiment in which we modified the ED problem to have a higher locality. We call this problem ED\_Local and show the results in \cref{fig:locality}. We can clearly see that the increased locality led to a reduction in necessary embedding size, which backs up our theoretical findings.

\section{Limitations \& Conclusion}
\textbf{Limitations.}
Although we show our results for representative efficient Transformers, it does not mean that our findings directly transfer to all models with similar designs. Further, despite our experiments indicating that Linear Transformers also benefit from locality,  it remains to prove whether the bound for the Linear Transformer can be tightened.

\textbf{Conclusion.} While the Sparse Transformer and Linear Transformer are expressive enough to solve general DP tasks, our findings indicate that they can not sustain their efficiency in the general case. This contradicts the anticipated efficiency gains, bringing their performance closer to that of standard Transformers, which maintain a constant model size.
The paradoxical nature of these efficient Transformers prompts a crucial question: under what conditions do these architectures become efficient? By delving into arithmetic expression evaluation and introducing the locality assumption, we identify scenarios where efficient Transformers can be efficient.
Our theoretical results find empirical support through extensive experiments on tasks like Arithmetic, Longest Increasing Subsequence (LIS), and Edit Distance (ED). The observed dependency between hidden dimension and problem scale for efficient Transformers, in contrast to the stability in standard Transformers, validates our theory.


\section*{Impact Statement}
This paper presents work whose goal is to advance the field of Machine Learning. There are many potential societal consequences of our work, none which we feel must be specifically highlighted here.

\section*{Acknowledgement}
{Liwei Wang is supported by National Science and Technology Major Project (2022ZD0114902) and National Science Foundation of China (NSFC62276005). Di He is supported by National Science Foundation of China (NSFC62376007). }

\bibliography{efficient_transformer}

\begin{thebibliography}{54}
\providecommand{\natexlab}[1]{#1}
\providecommand{\url}[1]{\texttt{#1}}
\expandafter\ifx\csname urlstyle\endcsname\relax
  \providecommand{\doi}[1]{doi: #1}\else
  \providecommand{\doi}{doi: \begingroup \urlstyle{rm}\Url}\fi

\bibitem[Aky{\"u}rek et~al.(2022)Aky{\"u}rek, Schuurmans, Andreas, Ma, and Zhou]{akyurek2022learning}
Aky{\"u}rek, E., Schuurmans, D., Andreas, J., Ma, T., and Zhou, D.
\newblock What learning algorithm is in-context learning? investigations with linear models.
\newblock \emph{arXiv preprint arXiv:2211.15661}, 2022.

\bibitem[Alberti et~al.(2023)Alberti, Dern, Thesing, and Kutyniok]{alberti2023sumformer}
Alberti, S., Dern, N., Thesing, L., and Kutyniok, G.
\newblock Sumformer: Universal approximation for efficient transformers.
\newblock In \emph{Topological, Algebraic and Geometric Learning Workshops 2023}, pp.\  72--86. PMLR, 2023.

\bibitem[Alman \& Song(2023)Alman and Song]{alman2023fast}
Alman, J. and Song, Z.
\newblock Fast attention requires bounded entries.
\newblock \emph{arXiv preprint arXiv:2302.13214}, 2023.

\bibitem[Beltagy et~al.(2020)Beltagy, Peters, and Cohan]{beltagy2020Longformer}
Beltagy, I., Peters, M.~E., and Cohan, A.
\newblock Longformer: The long-document transformer.
\newblock \emph{arXiv:2004.05150}, 2020.

\bibitem[Bhattamishra et~al.(2020)Bhattamishra, Ahuja, and Goyal]{bhattamishra2020ability}
Bhattamishra, S., Ahuja, K., and Goyal, N.
\newblock On the ability and limitations of transformers to recognize formal languages.
\newblock \emph{arXiv preprint arXiv:2009.11264}, 2020.

\bibitem[Brown et~al.(2020)Brown, Mann, Ryder, Subbiah, Kaplan, Dhariwal, Neelakantan, Shyam, Sastry, Askell, et~al.]{brown2020language}
Brown, T., Mann, B., Ryder, N., Subbiah, M., Kaplan, J.~D., Dhariwal, P., Neelakantan, A., Shyam, P., Sastry, G., Askell, A., et~al.
\newblock Language models are few-shot learners.
\newblock In \emph{Advances in neural information processing systems}, volume~33, pp.\  1877--1901, 2020.

\bibitem[Child et~al.(2019)Child, Gray, Radford, and Sutskever]{child2019generating}
Child, R., Gray, S., Radford, A., and Sutskever, I.
\newblock Generating long sequences with sparse transformers.
\newblock \emph{arXiv preprint arXiv:1904.10509}, 2019.

\bibitem[Choromanski et~al.(2021)Choromanski, Likhosherstov, Dohan, Song, Gane, Sarlos, Hawkins, Davis, Mohiuddin, Kaiser, et~al.]{choromanski2021rethinking}
Choromanski, K.~M., Likhosherstov, V., Dohan, D., Song, X., Gane, A., Sarlos, T., Hawkins, P., Davis, J.~Q., Mohiuddin, A., Kaiser, L., et~al.
\newblock Rethinking attention with performers.
\newblock In \emph{International Conference on Learning Representations}, 2021.

\bibitem[Dai et~al.(2023)Dai, Sun, Dong, Hao, Ma, Sui, and Wei]{dai2023can}
Dai, D., Sun, Y., Dong, L., Hao, Y., Ma, S., Sui, Z., and Wei, F.
\newblock Why can gpt learn in-context? language models implicitly perform gradient descent as meta-optimizers.
\newblock In \emph{ICLR 2023 Workshop on Mathematical and Empirical Understanding of Foundation Models}, 2023.

\bibitem[Dai et~al.(2019)Dai, Yang, Yang, Carbonell, Le, and Salakhutdinov]{dai2019transformer}
Dai, Z., Yang, Z., Yang, Y., Carbonell, J., Le, Q.~V., and Salakhutdinov, R.
\newblock Transformer-xl: Attentive language models beyond a fixed-length context.
\newblock \emph{arXiv preprint arXiv:1901.02860}, 2019.

\bibitem[Devlin et~al.(2019)Devlin, Chang, Lee, and Toutanova]{devlin2019bert}
Devlin, J., Chang, M.-W., Lee, K., and Toutanova, K.
\newblock {BERT}: Pre-training of deep bidirectional transformers for language understanding.
\newblock In \emph{Proceedings of the 2019 Conference of the North {A}merican Chapter of the Association for Computational Linguistics: Human Language Technologies, Volume 1 (Long and Short Papers)}, pp.\  4171--4186. Association for Computational Linguistics, 2019.

\bibitem[Elhage et~al.(2021)Elhage, Nanda, Olsson, Henighan, Joseph, Mann, Askell, Bai, Chen, Conerly, et~al.]{elhage2021mathematical}
Elhage, N., Nanda, N., Olsson, C., Henighan, T., Joseph, N., Mann, B., Askell, A., Bai, Y., Chen, A., Conerly, T., et~al.
\newblock A mathematical framework for transformer circuits.
\newblock \emph{Transformer Circuits Thread}, 1, 2021.

\bibitem[Feng et~al.(2023)Feng, Gu, Zhang, Ye, He, and Wang]{feng2023towards}
Feng, G., Gu, Y., Zhang, B., Ye, H., He, D., and Wang, L.
\newblock Towards revealing the mystery behind chain of thought: a theoretical perspective.
\newblock \emph{Advances in Neural Information Processing Systems}, 2023.

\bibitem[Furst et~al.(1984)Furst, Saxe, and Sipser]{furst1984parity}
Furst, M., Saxe, J.~B., and Sipser, M.
\newblock Parity, circuits, and the polynomial-time hierarchy.
\newblock \emph{Mathematical systems theory}, 17\penalty0 (1):\penalty0 13--27, 1984.

\bibitem[Garg et~al.(2022)Garg, Tsipras, Liang, and Valiant]{garg2022what}
Garg, S., Tsipras, D., Liang, P., and Valiant, G.
\newblock What can transformers learn in-context? a case study of simple function classes.
\newblock In \emph{Advances in Neural Information Processing Systems}, 2022.

\bibitem[Gu \& Dao(2023)Gu and Dao]{gu2023mamba}
Gu, A. and Dao, T.
\newblock Mamba: Linear-time sequence modeling with selective state spaces.
\newblock \emph{arXiv preprint arXiv:2312.00752}, 2023.

\bibitem[Hahn(2020)]{hahn2020theoretical}
Hahn, M.
\newblock Theoretical limitations of self-attention in neural sequence models.
\newblock \emph{Transactions of the Association for Computational Linguistics}, 8:\penalty0 156--171, 2020.

\bibitem[Hao et~al.(2022)Hao, Angluin, and Frank]{hao2022formal}
Hao, Y., Angluin, D., and Frank, R.
\newblock Formal language recognition by hard attention transformers: Perspectives from circuit complexity.
\newblock \emph{Transactions of the Association for Computational Linguistics}, 10:\penalty0 800--810, 2022.

\bibitem[Hendrycks \& Gimpel(2016)Hendrycks and Gimpel]{hendrycks2016gaussian}
Hendrycks, D. and Gimpel, K.
\newblock Gaussian error linear units (gelus).
\newblock \emph{arXiv preprint arXiv:1606.08415}, 2016.

\bibitem[Katharopoulos et~al.(2020)Katharopoulos, Vyas, Pappas, and Fleuret]{katharopoulos2020transformers}
Katharopoulos, A., Vyas, A., Pappas, N., and Fleuret, F.
\newblock Transformers are rnns: Fast autoregressive transformers with linear attention.
\newblock In \emph{International conference on machine learning}, pp.\  5156--5165. PMLR, 2020.

\bibitem[Keles et~al.(2023)Keles, Wijewardena, and Hegde]{keles2023computational}
Keles, F.~D., Wijewardena, P.~M., and Hegde, C.
\newblock On the computational complexity of self-attention.
\newblock In \emph{International Conference on Algorithmic Learning Theory}, pp.\  597--619. PMLR, 2023.

\bibitem[Kitaev et~al.(2020)Kitaev, Kaiser, and Levskaya]{kitaev2020reformer}
Kitaev, N., Kaiser, {\L}., and Levskaya, A.
\newblock Reformer: The efficient transformer.
\newblock \emph{arXiv preprint arXiv:2001.04451}, 2020.

\bibitem[Kojima et~al.(2022)Kojima, Gu, Reid, Matsuo, and Iwasawa]{kojima2022large}
Kojima, T., Gu, S.~S., Reid, M., Matsuo, Y., and Iwasawa, Y.
\newblock Large language models are zero-shot reasoners.
\newblock In \emph{Advances in Neural Information Processing Systems}, 2022.

\bibitem[Liu et~al.(2022)Liu, Ash, Goel, Krishnamurthy, and Zhang]{liu2022transformers}
Liu, B., Ash, J.~T., Goel, S., Krishnamurthy, A., and Zhang, C.
\newblock Transformers learn shortcuts to automata.
\newblock \emph{arXiv preprint arXiv:2210.10749}, 2022.

\bibitem[Liu et~al.(2023)Liu, Ash, Goel, Krishnamurthy, and Zhang]{liu2023transformers}
Liu, B., Ash, J.~T., Goel, S., Krishnamurthy, A., and Zhang, C.
\newblock Transformers learn shortcuts to automata.
\newblock In \emph{The Eleventh International Conference on Learning Representations}, 2023.

\bibitem[Loshchilov \& Hutter(2017)Loshchilov and Hutter]{loshchilov2017fixing}
Loshchilov, I. and Hutter, F.
\newblock Fixing weight decay regularization in adam.
\newblock 2017.

\bibitem[Luo et~al.(2021)Luo, Li, Cai, He, Peng, Zheng, Ke, Wang, and Liu]{luo2021stable}
Luo, S., Li, S., Cai, T., He, D., Peng, D., Zheng, S., Ke, G., Wang, L., and Liu, T.-Y.
\newblock Stable, fast and accurate: Kernelized attention with relative positional encoding.
\newblock In \emph{Advances in Neural Information Processing Systems}, volume~34, pp.\  22795--22807, 2021.

\bibitem[Merrill \& Sabharwal(2023)Merrill and Sabharwal]{merrill2023parallelism}
Merrill, W. and Sabharwal, A.
\newblock The parallelism tradeoff: Limitations of log-precision transformers.
\newblock \emph{Transactions of the Association for Computational Linguistics}, 2023.

\bibitem[Merrill et~al.(2022)Merrill, Sabharwal, and Smith]{merrill2022saturated}
Merrill, W., Sabharwal, A., and Smith, N.~A.
\newblock Saturated transformers are constant-depth threshold circuits.
\newblock \emph{Transactions of the Association for Computational Linguistics}, 10:\penalty0 843--856, 2022.

\bibitem[Nye et~al.(2022)Nye, Andreassen, Gur-Ari, Michalewski, Austin, Bieber, Dohan, Lewkowycz, Bosma, Luan, Sutton, and Odena]{nye2022show}
Nye, M., Andreassen, A.~J., Gur-Ari, G., Michalewski, H., Austin, J., Bieber, D., Dohan, D., Lewkowycz, A., Bosma, M., Luan, D., Sutton, C., and Odena, A.
\newblock Show your work: Scratchpads for intermediate computation with language models.
\newblock In \emph{Deep Learning for Code Workshop}, 2022.

\bibitem[Olsson et~al.(2022)Olsson, Elhage, Nanda, Joseph, DasSarma, Henighan, Mann, Askell, Bai, Chen, Conerly, Drain, Ganguli, Hatfield-Dodds, Hernandez, Johnston, Jones, Kernion, Lovitt, Ndousse, Amodei, Brown, Clark, Kaplan, McCandlish, and Olah]{olsson2022context}
Olsson, C., Elhage, N., Nanda, N., Joseph, N., DasSarma, N., Henighan, T., Mann, B., Askell, A., Bai, Y., Chen, A., Conerly, T., Drain, D., Ganguli, D., Hatfield-Dodds, Z., Hernandez, D., Johnston, S., Jones, A., Kernion, J., Lovitt, L., Ndousse, K., Amodei, D., Brown, T., Clark, J., Kaplan, J., McCandlish, S., and Olah, C.
\newblock In-context learning and induction heads.
\newblock \emph{Transformer Circuits Thread}, 2022.
\newblock https://transformer-circuits.pub/2022/in-context-learning-and-induction-heads/index.html.

\bibitem[OpenAI(2023)]{openai2023gpt4}
OpenAI.
\newblock Gpt-4 technical report.
\newblock \emph{arXiv preprint arXiv:2303.08774}, 2023.

\bibitem[Peng et~al.(2021)Peng, Pappas, Yogatama, Schwartz, Smith, and Kong]{peng2021random}
Peng, H., Pappas, N., Yogatama, D., Schwartz, R., Smith, N.~A., and Kong, L.
\newblock Random feature attention.
\newblock \emph{arXiv preprint arXiv:2103.02143}, 2021.

\bibitem[P{\'e}rez et~al.(2019)P{\'e}rez, Marinkovi{\'c}, and Barcel{\'o}]{perez2019turing}
P{\'e}rez, J., Marinkovi{\'c}, J., and Barcel{\'o}, P.
\newblock On the turing completeness of modern neural network architectures.
\newblock \emph{arXiv preprint arXiv:1901.03429}, 2019.

\bibitem[P{\'e}rez et~al.(2021)P{\'e}rez, Barcel{\'o}, and Marinkovic]{perez2021attention}
P{\'e}rez, J., Barcel{\'o}, P., and Marinkovic, J.
\newblock Attention is turing complete.
\newblock \emph{The Journal of Machine Learning Research}, 22\penalty0 (1):\penalty0 3463--3497, 2021.

\bibitem[Qiu et~al.(2020)Qiu, Ma, Levy, Yih, Wang, and Tang]{qiu2020blockwise}
Qiu, J., Ma, H., Levy, O., Yih, W.-t., Wang, S., and Tang, J.
\newblock Blockwise self-attention for long document understanding.
\newblock In \emph{Findings of the Association for Computational Linguistics: EMNLP 2020}, pp.\  2555--2565, 2020.

\bibitem[Radford et~al.(2019)Radford, Wu, Child, Luan, Amodei, Sutskever, et~al.]{radford2019language}
Radford, A., Wu, J., Child, R., Luan, D., Amodei, D., Sutskever, I., et~al.
\newblock Language models are unsupervised multitask learners.
\newblock \emph{OpenAI blog}, 1\penalty0 (8):\penalty0 9, 2019.

\bibitem[Roy et~al.(2021)Roy, Saffar, Vaswani, and Grangier]{roy2021efficient}
Roy, A., Saffar, M., Vaswani, A., and Grangier, D.
\newblock Efficient content-based sparse attention with routing transformers.
\newblock \emph{Transactions of the Association for Computational Linguistics}, 9:\penalty0 53--68, 2021.

\bibitem[Sanford et~al.(2023)Sanford, Hsu, and Telgarsky]{sanford2023representational}
Sanford, C., Hsu, D., and Telgarsky, M.
\newblock Representational strengths and limitations of transformers.
\newblock \emph{arXiv preprint arXiv:2306.02896}, 2023.

\bibitem[Sun et~al.()Sun, Dong, Huang, Ma, Xia, Xue, Wang, and Wei]{sun2307retentive}
Sun, Y., Dong, L., Huang, S., Ma, S., Xia, Y., Xue, J., Wang, J., and Wei, F.
\newblock Retentive network: A successor to transformer for large language models (2023).
\newblock \emph{URL http://arxiv. org/abs/2307.08621 v1}.

\bibitem[Tay et~al.(2020{\natexlab{a}})Tay, Bahri, Yang, Metzler, and Juan]{tay2020sparse}
Tay, Y., Bahri, D., Yang, L., Metzler, D., and Juan, D.-C.
\newblock Sparse sinkhorn attention.
\newblock In \emph{International Conference on Machine Learning}, pp.\  9438--9447. PMLR, 2020{\natexlab{a}}.

\bibitem[Tay et~al.(2020{\natexlab{b}})Tay, Dehghani, Abnar, Shen, Bahri, Pham, Rao, Yang, Ruder, and Metzler]{tay2020long}
Tay, Y., Dehghani, M., Abnar, S., Shen, Y., Bahri, D., Pham, P., Rao, J., Yang, L., Ruder, S., and Metzler, D.
\newblock Long range arena: A benchmark for efficient transformers.
\newblock \emph{arXiv preprint arXiv:2011.04006}, 2020{\natexlab{b}}.

\bibitem[Tay et~al.(2022)Tay, Dehghani, Bahri, and Metzler]{tay2022efficient}
Tay, Y., Dehghani, M., Bahri, D., and Metzler, D.
\newblock Efficient transformers: A survey, 2022.

\bibitem[Touvron et~al.(2023)Touvron, Martin, Stone, Albert, Almahairi, Babaei, Bashlykov, Batra, Bhargava, Bhosale, et~al.]{touvron2023llama2}
Touvron, H., Martin, L., Stone, K., Albert, P., Almahairi, A., Babaei, Y., Bashlykov, N., Batra, S., Bhargava, P., Bhosale, S., et~al.
\newblock Llama 2: Open foundation and fine-tuned chat models.
\newblock \emph{arXiv preprint arXiv:2307.09288}, 2023.

\bibitem[Vaswani et~al.(2017)Vaswani, Shazeer, Parmar, Uszkoreit, Jones, Gomez, Kaiser, and Polosukhin]{vaswani2017attention}
Vaswani, A., Shazeer, N., Parmar, N., Uszkoreit, J., Jones, L., Gomez, A.~N., Kaiser, {\L}., and Polosukhin, I.
\newblock Attention is all you need.
\newblock \emph{Advances in neural information processing systems}, 30, 2017.

\bibitem[Von~Oswald et~al.(2023)Von~Oswald, Niklasson, Randazzo, Sacramento, Mordvintsev, Zhmoginov, and Vladymyrov]{von2023transformers}
Von~Oswald, J., Niklasson, E., Randazzo, E., Sacramento, J., Mordvintsev, A., Zhmoginov, A., and Vladymyrov, M.
\newblock Transformers learn in-context by gradient descent.
\newblock In \emph{International Conference on Machine Learning}, pp.\  35151--35174. PMLR, 2023.

\bibitem[Vyas et~al.(2020)Vyas, Katharopoulos, and Fleuret]{vyas2020fast}
Vyas, A., Katharopoulos, A., and Fleuret, F.
\newblock Fast transformers with clustered attention.
\newblock In \emph{Advances in Neural Information Processing Systems}, volume~33, pp.\  21665--21674, 2020.

\bibitem[Wang et~al.(2020)Wang, Li, Khabsa, Fang, and Ma]{wang2020linformer}
Wang, S., Li, B.~Z., Khabsa, M., Fang, H., and Ma, H.
\newblock Linformer: Self-attention with linear complexity.
\newblock \emph{arXiv preprint arXiv:2006.04768}, 2020.

\bibitem[Wei et~al.(2022{\natexlab{a}})Wei, Chen, and Ma]{wei2022statistically}
Wei, C., Chen, Y., and Ma, T.
\newblock Statistically meaningful approximation: a case study on approximating turing machines with transformers.
\newblock \emph{Advances in Neural Information Processing Systems}, 35:\penalty0 12071--12083, 2022{\natexlab{a}}.

\bibitem[Wei et~al.(2022{\natexlab{b}})Wei, Wang, Schuurmans, Bosma, Xia, Chi, Le, Zhou, et~al.]{wei2022chain}
Wei, J., Wang, X., Schuurmans, D., Bosma, M., Xia, F., Chi, E., Le, Q.~V., Zhou, D., et~al.
\newblock Chain-of-thought prompting elicits reasoning in large language models.
\newblock \emph{Advances in Neural Information Processing Systems}, 35:\penalty0 24824--24837, 2022{\natexlab{b}}.

\bibitem[Weiss et~al.(2021)Weiss, Goldberg, and Yahav]{weiss2021thinking}
Weiss, G., Goldberg, Y., and Yahav, E.
\newblock Thinking like transformers.
\newblock In \emph{International Conference on Machine Learning}, pp.\  11080--11090. PMLR, 2021.

\bibitem[Yao et~al.(2021)Yao, Peng, Papadimitriou, and Narasimhan]{yao2021self}
Yao, S., Peng, B., Papadimitriou, C., and Narasimhan, K.
\newblock Self-attention networks can process bounded hierarchical languages.
\newblock In \emph{Proceedings of the 59th Annual Meeting of the Association for Computational Linguistics and the 11th International Joint Conference on Natural Language Processing (Volume 1: Long Papers)}, pp.\  3770--3785, 2021.

\bibitem[Yun et~al.(2019)Yun, Bhojanapalli, Rawat, Reddi, and Kumar]{yun2019transformers}
Yun, C., Bhojanapalli, S., Rawat, A.~S., Reddi, S.~J., and Kumar, S.
\newblock Are transformers universal approximators of sequence-to-sequence functions?
\newblock \emph{arXiv preprint arXiv:1912.10077}, 2019.

\bibitem[Zhou et~al.(2023)Zhou, Sch{\"a}rli, Hou, Wei, Scales, Wang, Schuurmans, Cui, Bousquet, Le, and Chi]{zhou2023leasttomost}
Zhou, D., Sch{\"a}rli, N., Hou, L., Wei, J., Scales, N., Wang, X., Schuurmans, D., Cui, C., Bousquet, O., Le, Q.~V., and Chi, E.~H.
\newblock Least-to-most prompting enables complex reasoning in large language models.
\newblock In \emph{The Eleventh International Conference on Learning Representations}, 2023.

\end{thebibliography}
\bibliographystyle{icml2024}

\newpage
\appendix
\onecolumn
\section{Technical Lemmas}
\label{sec:lemmas}

In this section, we propose some lemmas about the expressive power of the MLP, the linear transformer, and the sparse transformer. Due to the similarity between the efficient and the standard Transformer, we base some ideas and constructions on previous work \citep{feng2023towards}.

\subsection{Lemmas for MLP}

In the previous work \citep{feng2023towards}, the authors investigated the expressive power of the MLP with GeLU activation function. They showed that a two-layer MLP with GeLU activation and a fixed number of log-precision neurons can perform various tasks such as multiplication, linear transformation, and selection. Building on their work, our proof is more concise and clear. We will restate the relevant lemmas below, and refer to the appendices of \citet{feng2023towards} for the proofs.

\begin{lemma}[From \citet{feng2023towards}]
\label{lemma:MLP_multip}
    Let $f:\mathbb R^2\to\mathbb R$ be a two-layer MLP with GeLU activation, and the hidden dimension is 4. Then, for any $\epsilon > 0$ and $M > 0$, there exist MLP parameters with $\ell_{\infty}$ norm upper bounded by $O(\mathrm{poly}(M,1/\eps))$ such that $|f(a, b) - ab| \leq \epsilon$ holds for all $a, b \in [-M, M]$.
\end{lemma}

\begin{lemma}[From \citet{feng2023towards}]
\label{lemma:MLP_relu}
    Let $\vg:\mathbb R^{d_1}\to \mathbb R^{d_2}$ be a two-layer MLP with $\mathrm{ReLU}$ activation, and all parameter values are upper bounded by $M$. Then, for any $\epsilon>0$, there exists a two-layer MLP $\vf$ of the same size with $\mathrm{GeLU}$ activation and parameters upper bounded by $O(\mathrm{poly}(M, 1/\epsilon))$ in the $\ell_{\infty}$ norm, such that for all $\vx\in\mathbb R^{d_1}$, we have $\|\vf(\vx)-\vg(\vx)\|_\infty\leq \epsilon$.
\end{lemma}

\begin{lemma}[From \citet{feng2023towards}]
\label{lemma:MLP_lin}
    Let $\vf:\mathbb R^{d_1}\to \mathbb R^{d_2}$ be a two-layer MLP with $\mathrm{GeLU}$ activation, and the hidden dimension is $2d_2$. Let $\mW\in\mathbb R^{d_2\times d_1}$ be any matrix and denote $M=\max_{ij}|W_{ij}|$. Then, for any $\epsilon>0$, there exist MLP parameters with $\ell_{\infty}$ norm bounded by $O(\mathrm{poly}(M,1/\epsilon))$, such that for any $\vx\in \mathbb{R}^{d_1}$, we have $\|\vf(\vx)-\mW \vx\|_\infty\leq \epsilon$.
\end{lemma}

\begin{lemma}[From \citet{feng2023towards}]
\label{lemma:MLP_select}
    Define the selection function $\vg:\mathbb R^{d}\times \mathbb R^{d}\times \mathbb R\to \mathbb R^{d}$ as follows:
    \begin{equation}
        \vg(\vx,\vy,t)=\left\{\begin{array}{cc}
            \vx & \text{if }t\ge 0, \\
            \vy & \text{if }t< 0.
        \end{array}\right.
    \end{equation}
    Let $\vf:\mathbb R^{d}\times \mathbb R^{d}\times \mathbb R\to \mathbb R^{d}$ be a two-layer MLP with GeLU activation, and the hidden dimension is $2d+2$. Then, for any $\epsilon>0$, $\alpha>0$, and $M>0$, there exist MLP parameters with $\ell_{\infty}$ norm bounded by $O(\mathrm{poly}(M,1/\alpha, 1/\epsilon))$, such that for all $\vx\in[-M,M]^{d}$, $\vy\in [-M,M]^{d}$, and $t\in [-\infty,-\alpha]\cup[\alpha,+\infty]$, we have $\|\vf(\vx,\vy,t)-\vg(\vx,\vy,t)\|_\infty\leq \epsilon$.
\end{lemma}

\begin{lemma}
\label{lemma:MLP_onehot1}
    For integer $i$, define $S_i=\left[i-\frac13, i+\frac13\right]$. Let $n\in\sZ$ and $D=\cup_{i=0}^n S_i$. 
    Define the function $\vg:D\to\sR^n$ as follows: $\vf(x)=\vm_i$ if $x\in S_i$, where $\vm_i$ has its first $i$ entries $1$ and the rest entries $0$.
    Let $\vf:\sR\to\sR^n$ be a two-layer MLP with GeLU activation, and the hidden dimension is $2n$. 
    Then, for any $\eps>0$, $\alpha>0$, there exists MLP parameters with $l_\infty$ norm bounded by $O(\mathrm{poly}(n,1/\eps))$, such that for all $x\in D$, we have $\|\vf(x)-\vg(x)\|_\infty\leq\eps$.
\end{lemma}
\begin{proof}
    Notice that the $k$-th entry of $\vm_j$ is
    \begin{align}
        \sI[k\leq j]=\mathrm{ReLU}\left[-2k+2\left(j+\frac34\right)\right]-\mathrm{ReLU}\left[-2k+2\left(j+\frac14\right)\right]
    \end{align}
    Thus each entry of $\vm_j$ can be implemented by an MLP with hidden dimension $2$. By Lemma \ref{lemma:MLP_relu}, we can perform both tasks using an MLP with GeLU activation, with hidden dimension $2n$. 
\end{proof}

\begin{lemma}
\label{lemma:MLP_onehot2}
    For integer $i$, define $S_i=\left[i-\frac13, i+\frac13\right]$. Let $n\in\sZ$ and $D=\cup_{i=0}^n S_i$. 
    Define the function $\vg:D\to\sR^n$ as follows: $\vf(x)=\vn_i$ if $x\in S_i$, where $\vn_i$ has its first $i-1$ entries $0$ and the rest entries $1$.
    Let $\vf:\sR\to\sR^n$ be a two-layer MLP with GeLU activation, and the hidden dimension is $2n$. 
    Then, for any $\eps>0$, $\alpha>0$, there exists MLP parameters with $l_\infty$ norm bounded by $O(\mathrm{poly}(n,1/\eps))$, such that for all $x\in D$, we have $\|\vf(x)-\vg(x)\|_\infty\leq\eps$.
\end{lemma}
\begin{proof}
    Similarly to the previous proof, the $k$-th entry of $\vn_j$ is
    \begin{align}
        \sI[k\geq j]=\mathrm{ReLU}\left[2k-2\left(j-\frac34\right)\right]-\mathrm{ReLU}\left[2k-2\left(j-\frac14\right)\right]
    \end{align}
    which indicates that each entry of $\vn_j$ can be implemented by an MLP with hidden dimension $2$. By Lemma \ref{lemma:MLP_relu}, we can perform both tasks using an MLP with GeLU activation, with hidden dimension $2n$. 
\end{proof}

\begin{lemma}
\label{lemma:MLP_devision}
For integer $i$, define $S_i=\left[i-\frac14, i+\frac14\right]$.
Given positive integer $n$ and integer $i\in[1,n^2]$. There exists a two-layer MLP $\vf:\sR\to\sR^3$ with GeLU activation and hidden dimension $O(n)$, such that for any integer $i\in[1,n^2]$ and $x\in S_i$, $\vf(x)=(\lceil i/n\rceil$, $\lceil i/n\rceil-1$, $n-(\lceil i/n\rceil\times n-i))$. Here, $n-(\lceil i/n\rceil\times n-i)=i\text{ mod }n$ if $i\text{ mod }n\ne 0$, and $n-(\lceil i/n\rceil\times n-i)=n$ if $i\text{ mod }n= 0$.
\end{lemma}
\begin{proof}
Since $i\in\{1,2,\cdots,n^2\}$, we can get $\lceil i/n\rceil\in\{1,2,\cdots,n\}$. By the proposition that $i$ is an integer, we can use an MLP with hidden dimension $2n$ and ReLU activation to calcualte $\lceil i/n\rceil$ as following:
\begin{align}
\lceil i/n\rceil = \sum_{j=1}^n \sI[i\leq jn]=\sum_{j=1}^n \mathrm{ReLU}\left[-2i+2\left(jn+\frac34\right)\right]
-\mathrm{ReLU}\left[-2i+2\left(jn+\frac14\right)\right]
\end{align}
Thus we can calculate $\lceil i/n\rceil$, $\lceil i/n\rceil-1$, $n-(\lceil i/n\rceil\times n-i)$ task by an MLP with ReLU activation with hidden dimension $O(n)$.
By Lemma \ref{lemma:MLP_relu}, we can finish these tasks by an MLP with GeLU activation with hidden dimension $O(n)$.
\end{proof}

\subsection{Lemmas for Linear Transformer}

\begin{lemma}
\label{lemma: linear transformer as RNN}
The attention module in Linear Transformer
\begin{align}
\label{eq:linear_attention_origin}
    \mathrm{Attn}(\vx_i)=\frac{\sum_{j\leq i}\phi(\mW_K\vx_j)^\top\phi(\mW_Q\vx_i) \mW_V\vx_i}{\sum_{j\leq i}\phi(\mW_K\vx_j)^\top\phi(\mW_Q\vx_i)}
\end{align}
can be implemented as follows:
\begin{align}
\label{eq:linear_attention_RNN}
    &\vs_0=\mathbf{0},\;\vz_0=\mathbf{0}\\
    &\vs_i=\vs_{i-1}+\phi(\mW_K\vx_i)(\mW_V\vx_i)^\top\\
    &\vz_i=\vz_{i-1}+\phi(\mW_K\vx_i)\\
    &\mathrm{Attn}(\vx_i)^\top=\frac{\phi(\mW_Q\vx_i)^\top\vs_i}{\phi(\mW_Q\vx_i)^\top\vz_i}
\end{align}
Since this implementation is quite similar to RNNs, we called $\vs,\vz$ in our implementation hidden states.
\end{lemma}

Now, we introduce the AGG operation for Linear Transformer. Let $m$ be an integer and let $\vx_1,\vx_2,\cdots,\vx_{m^2}$ be a sequence of vectors where $\vx_i=(\hat{\vx}_i,\ve_{\lceil i/m\rceil},\ve_{q_i},1)\in [-M,M]^{d+m+1}$, $\hat{\vx}_i\in\sR^d$, and $\ve_{\lceil i/m\rceil},\ve_{q_i}$ are one-hot vectors in $\sR^m$, and $M$ is a large constant.
Let $\mathcal{S}_i=\{j\in\sZ:q_im-(m-1)\leq j\leq \min(q_im,i)\}$.
Define the AGG operation as follows: The output is a sequence of vectors $\vu_1,\cdots,\vu_{m^2}$ with $\vu_i=\mathrm{mean}_{j\in\mathcal{S}_i}\hat\vx_j$. The output is undefined when $\mathcal{S}_i=\varnothing$.

\begin{lemma}
\label{lemma:linear transformer aggregation}
For any $0<\epsilon\leq M$, there exists an linear attention layer with hidden dimension $O(\max(d,m))$ and one causal attention head that can approximate the AGG operation defined above. Specifically, suppose the attention output are $\vo_1,\cdots,\vo_{m^2}$, then we have $\|\vo_i-\vu_i\|_\infty\leq \epsilon$ for all $i\in[m^2]$. 
Moreover, the $l_{\infty}$ norm of attention parameters are bounded by $O(\mathrm{poly}(\log(M),\log(m),\log(1/\epsilon))$. 
\end{lemma}
\begin{proof}
Without loss of generosity, we can assume $d=m$ (otherwise, we can do some padding with $0$). We will provide a proof based on the implement of Lemma \ref{lemma: linear transformer as RNN}.

We construct the query, key and value vectors as follows:
\begin{itemize}
    \item Query: $\vq_i=\mu\ve_{q_i}-\mu$
    \item Key: $\vk_i=\mu\ve_{\lceil i/m\rceil}-\mu$
    \item Value: $\vv_i=\hat{\vx}_i$
\end{itemize}
where $\mu>0$ is a constant defined later. This can be achieved by setting appropriate $\mW_Q,\mW_K,\mW_V$.

Notice that $j\in\mathcal{S}_i$ if and only if $j\leq i$ and $\lceil j/m\rceil=q_i$.
By the construction of $\vq_i,\vk_i,\vv_i$, we can get 
\begin{align*}
a_{ij}:=\vq_i^\top\vk_j=\begin{cases}
    1+(m-1)\exp(-\mu),\;&j\in\mathcal{S}_i\\
    2\exp(-\mu)+(m-2)\exp(-2\mu),\;&j\notin\mathcal{S}_i
\end{cases}
\end{align*}
By taking $\mu=\ln\frac{2Mm^3}{\epsilon}$, which is bounded by $O(\mathrm{poly}(\log(M),\log(m),\log(1/\epsilon))$, we have 
\begin{align*}
\sum_{j\notin\mathcal{S}_i} a_{ij}&\leq 
\frac{(m^2-|\gS_i|)[2\exp(-\mu)+(m-2)\exp(-2\mu)]}{(m^2-|\gS_i|)[2\exp(-\mu)+(m-2)\exp(-2\mu)]+|\gS_i|[1+(m-1)\exp(-\mu)]}\\
&=\frac{1}{1+\frac{|\gS_i|}{m^2-|\gS_i|}\cdot\frac{1+(m-1)\exp(-\mu)}{2\exp(-\mu)+(m-2)\exp(-2\mu)}}\\
&\leq \frac{m^2-|\gS_i|}{|\gS_i|}\cdot\frac{2\exp(-\mu)+(m-2)\exp(-2\mu)}{1+(m-1)\exp(-\mu)}\\
&\leq m^2\cdot m\exp(-\mu)=m^3\exp(-\mu)=\frac{\epsilon}{2M}
\end{align*}
Similarly, for $j\in\gS_i$, we have 
\begin{align*}
&\quad \left|a_{ij}-\frac{1}{|\gS_i|}\right|\\
&\leq \frac{1}{|\gS_i|}-\frac{1+(m-1)\exp(-\mu)}{(m^2-|\gS_i|)[2\exp(-\mu)+(m-2)\exp(-2\mu)]+|\gS_i|[1+(m-1)\exp(-\mu)]}\\
&=\frac{1}{|\gS_i|}-\frac{1}{(m^2-|\gS_i|)\frac{2\exp(-\mu)+(m-2)\exp(-2\mu)}{1+(m-1)\exp(-\mu)}+|\gS_i|}\\
&=\frac{(m^2-|\gS_i|)\frac{2\exp(-\mu)+(m-2)\exp(-2\mu)}{1+(m-1)\exp(-\mu)}}{|\gS_i|\left[(m^2-|\gS_i|)\frac{2\exp(-\mu)+(m-2)\exp(-2\mu)}{1+(m-1)\exp(-\mu)}+|\gS_i|\right]}\\
&\leq \frac{1}{|\gS_i|}(m^2-|\gS_i|)\cdot [2\exp(-\mu)+(m-2)\exp(-2\mu)]\\
&\leq \frac{1}{|\gS_i|}m^2\cdot m\exp(-\mu)=\frac{1}{|\gS_i|}m^3\exp(-\mu)=\frac{\epsilon}{2M|\gS_i|}
\end{align*}
We thus obtain
\begin{align*}
\|\vo_i-\vu_i\|_\infty=\left\|\sum_j a_{ij}\hat\vx_j-\frac{1}{|\gS_i|}\sum_{j\in\gS_i}\hat\vx_j\right\|_\infty
\leq \max_j \|\hat\vx_j\|_\infty\cdot 
\left(\sum_{j\notin\gS_i} a_{ij}+\sum_{j\in\gS_i}\left|a_{ij}-\frac{1}{|\gS_i|}\right|\right)\le \epsilon
\end{align*}
which concludes our proof.
\end{proof}

\subsection{Lemmas for Sparse Transformer}

The previous work by \citet{feng2023towards} studied the expressive power of the standard attention layer and introduced a basic operation, COPY, that can be implemented by the attention layer. In this section, we introduce an updated version of the COPY operation for the sparse transformer and demonstrate that the sparse transformer module is capable of implementing this operation. 

Firstly, we will enumerate the elements and their respective notations that will be employed to define the COPY operation.
\label{op:copy}

\begin{itemize}
    \item A sequence of vectors $\vx_1,\vx_2,\cdots,\vx_n\in\mathbb{R}^{d+2}$, where $\vx_i=[\vv_i,r_i,i\bmod B,\lceil\frac{i}{B}\rceil,1]$, $\vx_i\in \mathbb{R}^d$, $r_i\in \mathbb{R}$, $B\in \mathbb{Z}$.
    \item Three matrices $\mK,\mQ,\mV\in \mathbb{R}^{d^\prime\times(d+2)}$, where $\|V\|_\infty\leq 1$
    \item The attention sets $\gS^\prime_i$ for each $i$, where $\gS_i^\prime$ is the indices of the tokens attended by the $i$-th token.
    \item Two real number $\rho,\delta>0$
\end{itemize}

Denoting that $\vk_i=\mK\vx_i,\vq_i=\mQ\vx_i,\vv_i=\mV\vx_i\in \mathbb{R}^{d^\prime}$, we can define the matching set $\gS_i$ for the $i$ as $\gS_i=\{j<i\ |\ \vq_i\cdot\vk_j<\rho\}\cap\gS_i^\prime$. Then the output of the COPY operation $\vu_1,\vu_2,\cdots,\vu_n$ is defined as $\vu_i = \vv_{\operatorname{pos(i)}}$, where $\operatorname{pos}(i)=\operatorname{argmax}_{j\in\gS_i} r_j$. Intuitively, the COPY operation copies the embedding of the matching token with the highest rating. Note that the output of the COPY operation can be anything when $\gS_i$ is empty. 

Moreover, we make the following regularity assumption:
\begin{assumption}
\label{ass:regular}
    The input sequence  $\vx_1,\vx_2,\cdots,\vx_n$ and the matrices $\mK,\mQ,\mV$ satisfy the following condition:
    \begin{itemize}
        \item For any $i,j\in[n]$, either $|q_i\cdot k_j | \leq \rho$ or $q_i \cdot k_j \leq -\delta$
        \item For any $i,j\in[n]$, either $i=j$ or $|r_i-r_j| \geq \delta$
        \item $\|\mV\|_{\infty}\leq 1$
    \end{itemize}
\end{assumption}
\cref{ass:regular} ensures that only one token with the highest rating is matched, and there is a substantial difference between this token and others. The subsequent lemma demonstrates that the sparse transformer module can effectively implement the COPY operation.

\begin{lemma}[COPY Operation of Standard Transformer, From \citet{feng2023towards}]
    \label{lemma:standard_copy} 
    Given that \cref{ass:regular} holds with $\rho\le\frac{\delta^2}{8M}$, there exists a standard transformer module defined in \cref{eq:sparseattention} with one attention layer, block size $B$, embedding size $O(d)$, and one attention head that can approximate the COPY operation defined above. Specifically, for any sequence of vectors $\vx_1, \vx_2, \cdots, \vx_n$, denote the corresponding output of the attention layer as $\vo_1, \vo_2, \cdots, \vo_n$. Then, we have $\|o_i-u_i\|_{\infty}\le\epsilon$ for all $i\in [n]$ with $\gS_i\neq \emptyset$. Furthermore, the $\ell_\infty$ norm of attention parameters is bounded by $O(\mathrm{poly}(M,1/\delta,\log(n),\log(1/\epsilon)))$.
\end{lemma}

Now, we will prove a special form of the COPY operation for the Sparse Transformer, such that given an index of a token, COPY the embedding of the token with the given index.

\begin{lemma}[COPY Operation of Sparse Transformer]
    \label{lemma:sparse_copy} 
    Given that \cref{ass:regular} holds with $\rho\le\frac{\delta^2}{8M}$, there exists a sparse transformer module defined in \cref{eq:sparseattention} with three attention layer, block size $B$, embedding size $O(d)$, and one attention head that can approximate the COPY operation defined above. Specifically, for any sequence of vectors $\vx_1, \vx_2, \cdots, \vx_n$, denote the corresponding output of the attention layer as $\vo_1, \vo_2, \cdots, \vo_n$. Then, we have $\|o_i-u_i\|_{\infty}\le\epsilon$ for all $i\in [n]$ with $\gS_i\neq \emptyset$. Furthermore, the $\ell_\infty$ norm of attention parameters is bounded by $O(\mathrm{poly}(M,1/\delta,\log(n),\log(1/\epsilon)))$.
\end{lemma}

The proof for this version is an extension of the proof for a Lemma described in \citet{feng2023towards}. For the reader's convenience, we provide a concise proof as follows.

\begin{proof}
    \textbf{Layer 1.} In our formulation, the input embedding $\vx_i$ has the form $\vx_i=(\vv_i,p_\text{target},i\bmod B,\lceil\frac{i}{B} \rceil,1)$, where $p_\text{target}$ is the target index of the COPY operation. The primary objective of the first layer is to utilize its MLP to transform the input embedding. The target output of this layer is 
    \begin{equation*}
    [\vv_{i,1},\vv_{i,2},\cdots,\vv_{i,B},p_\text{target},i\bmod B,\lceil\frac{i}{B} \rceil,1],\text{\ \ where \ } \vv_{i,j}=\begin{cases}
        \mathbf{0}&\text{, when $j\neq i\bmod B+1$}\\
        \vv_i&\text{, when $j= i\bmod B+1$}
    \end{cases}
    \end{equation*}
    Note that $v_i$ is bound by $[-M,M]$. Therefore, we can express $\vv_{i,j}$ as
    \begin{equation*}
        \vv_{i,j}=\mathrm{ReLU}\Big(\mathrm{ReLU}(\vv_i+M)-M\cdot\big(1+\cdot\mathrm{ReLU}(j-i\bmod B)+\cdot\mathrm{ReLU}(i\bmod B-j)\big)\Big)
    \end{equation*}
    According to \cref{lemma:MLP_relu}, an MLP with $\mathrm{GeLU}$ as the activation function can accomplish this task with arbitrarily small error. Therefore, we can write the output of this layer as $\vx_i^{(1)}=[\vv_{i,1},\vv_{i,2},\cdots,\vv_{i,B},p_\text{target},i\bmod B,\lceil\frac{i}{B} \rceil,1]$.

    \textbf{Layer 2.} The main task of the second layer is to sum over the embedding of previous $B$ tokens. In this layer, the attention head just pays uniform attention to the previous $B$ tokens. Note that, in the first layer, we distribute the embeddings of different tokens on the different dimensions of the embedding. Therefore, the output of the attention layer is $\frac{1}{B}\cdot[\vv_{(i-B+ i\bmod B+1)},\vv_{(i-B+ i\bmod B+2)},\cdots,\vv_{i},\cdots \vv_{(i-B+i\bmod B)},p_\text{target},i\bmod B,\lceil\frac{i}{B} \rceil,1]$.
    Then, by using the MLP to multiply the number of tokens $B$, we can get the embeddings of all tokens. Moreover, in this layer, according to Lemma~\ref{lemma:MLP_devision} we can use the MLP to calculate the index of the block of the target token as $\lceil \frac{p_\text{target}}{B}\rceil$. The output of this layer as
    \begin{equation*}
        \vx_i^{(2)}=[\vv_{(i-B+ i\bmod B+1)},\vv_{(i-B+ i\bmod B+2)},\cdots,\vv_{i},\cdots \vv_{(i-B+i\bmod B)},p_\text{target},i\bmod B,\lceil\frac{i}{B} \rceil,\lceil \frac{p_\text{target}}{B}\rceil,1].
    \end{equation*}

    \textbf{Layer 3.} The main task of the second layer is to COPY the embedding of the token in the corresponding block and select the target embedding. In this layer, the attention head just COPY the embedding of the token in the corresponding block. According to \cref{lemma:standard_copy}, the attention layer can complete this task with arbitrarily small errors. Finally, we use the MLP to select the target embedding from the copied embedding and obtain the final output.
\end{proof}

\section{Proofs of the theorems in \cref{sec:DP}}
\label{appendix:proof_secDP}
For ease of reading, we reclaim some important assumptions here, which are natural and mild.

\begin{assumption}
\label{ass:dp_approximation}
    Each function $f$, $\vg$, $\vh$ and $u$ defined in CoT for DP can be approximated by constant size MLP with GeLU activation.
\end{assumption}
\begin{assumption}
\label{ass:dp_next_index}
    The state transition function $F$, where the inputs are the token number and current state, and the outputs are the next state, can be approximated with MLP and GeLU function. 
\end{assumption}

In addition we replace all residual connections in attention layers and MLP with concatenation, which doesn't change the expressive power of model architecture.

\subsection{Proof of the \cref{thm:general_upper_bound}}
\label{appendix:upper_dp}
In this subsection, we provide a complete proof of \cref{thm:general_upper_bound}. For clarity, we restate the theorem for the linear transformer and sparse transformer individually and provide their respective proofs.
\begin{theorem}
Consider any DP problem satisfying the same condition as in \cref{thm:baseline}. Given any integer $n>0$, let $L$ be the length of the output sequence when the input sequence length is $n$. Then, there is a (log-precision) linear transformer with a constant depth $M$, a constant number of attention heads $H$, and a hidden dimension $D=\mathrm{O}(\sqrt L)$ that can generate the correct output for all inputs $\vs$ of length $n$.
\end{theorem}

\begin{proof}
Denote $W=\lceil\sqrt{L}\rceil$. In this formulation, the embeddings are $\vx^{(0)}_k=(\ve^{\text{input}}_k, \ve^{\text{state}}_k, \ve^{\text{dp}}_k, \ve^{\text{answer}}_k,k,1)$, where $\ve^{\text{input}}_k, \ve^{\text{state}}_k, \ve^{\text{dp}}_k, \ve^{\text{answer}}_k$ corresponds to input token, DP state, DP value, final result. The embeddings are $\mathbf{0}$ if it's not defined. We construct the layers as follows.

\textbf{Block 1}. According to Assumption \ref{ass:dp_next_index}, we can use constant layers of MLPs to obtain $\ve^{\text{next\_state}}_k$ with input $\ve^{\text{state}}_k$. This can be done by setting zero weight matrices in attention layers and discard attention output in linear projection of MLP.
The output of this block is $\vx^{(1)}_k=(\ve^{\text{input}}_k, \ve^{\text{state}}_k,\ve^{\text{next\_state}}_k, \ve^{\text{dp}}_k,k,1)$.

\textbf{Block 2}. 
The second layer of the Transformer finishes the following tasks:
\begin{itemize}
    \setlength{\itemsep}{0pt}
    \item Calculate $\vh(\ve^{\text{next\_state}}_k),\vg(\ve^{\text{next\_state}}_k)$. 
    \item Calculate $f^{\text{state}}_k$ as the indicator variable of $\ve^{\text{state}}_k$ is the last state.
    \item Calculate $(\vh(\ve^{\text{next\_state}}_k))^2,(\vg(\ve^{\text{next\_state}}_k))^2,(\ve^{\text{state}}_k)^2$, which are element-wise square operations, and $k^2$.
\end{itemize}
The first task can be implemented by several layers of MLPs by Assumption \ref{ass:dp_approximation}. To perform the second task, we can check whether $\ve^{\text{state}}_k\neq\mathbf 0$ and $\ve^{\text{next\_state}}_k=\mathbf 0$. The last task can be done with an MLP using Lemma \ref{lemma:MLP_multip}. The output of this block is
\begin{align*}
    \vx^{(2)}_k=(&\ve^{\text{input}}_k, \ve^{\text{state}}_k,\ve^{\text{next\_state}}_k, \ve^{\text{dp}}_k,\ve^{\text{sep}}_k,\vh(\ve^{\text{next\_state}}_k),\vg(\ve^{\text{next\_state}}_k),\\
    &(\vh(\ve^{\text{next\_state}}_k))^2,(\vg(\ve^{\text{next\_state}}_k))^2,(\ve^{\text{state}}_k)^2,f^{\text{state}}_k,  k,k^2,1).
\end{align*}

\textbf{Block 3}. The third block of the Transformer uses $K+J$ heads to perform the following tasks where $K$ and $J$ are defined in (\ref{eq:dp_transition}), refers to the transition in DP:
\begin{itemize}
    \setlength{\itemsep}{0pt}
    \item Get inputs of $s_{g_1(i)},\cdots,s_{g_J(i)}$ where $i$ corresponds to $\ve^{\text{next\_state}}_k$. 
    \item Get DP value of $\mathsf{dp}(h_1(i)),\cdots,\mathsf{dp}(h_K(i))$ for $i$ corresponds to $\ve^{\text{next\_state}}_k$. 
    \item Calculate $\ve^{\text{next\_dp}}_k$, i.e., the DP value of the next state, $\ve^{\text{next\_state}}_k$.
\end{itemize}

    To perform the first two tasks, we need to first calculate the absolute position of the embeddings we want.
    This can be implemented by several MLPs by Assumption \ref{ass:dp_approximation} and problem size $n$.
    Suppose the absolute positions are $\hat g_1,\cdots,\hat g_J,\hat h_1,\cdots,\hat h_K$. 
    Then we can do the following tasks for $t=\hat g_1,\cdots,\hat g_J,\hat h_1,\cdots,\hat h_K$:
    \begin{enumerate}
        \item Calculate $\lceil t/W\rceil$, $W-(\lceil t/W\rceil\times W-t)$ by Lemma \ref{lemma:MLP_devision}.
        \item Calculate $\ve_s\in\sR^W$ for $s=\lceil t/W\rceil,W-(\lceil t/W\rceil\times W-t)$ by Lemma \ref{lemma:MLP_onehot1}, \ref{lemma:MLP_onehot2} and the fact that $\ve_i=\vm_i\odot \vn_i$.
    \end{enumerate}
    We also calculate $\ve_s\in\sR^W$ for $s=\lceil k/W\rceil,W-(\lceil k/W\rceil\times W-k)$ by Lemma \ref{lemma:MLP_onehot1}, \ref{lemma:MLP_onehot2}. Then we can get 
    \begin{align*}
        (\sI[W-(\lceil k/W\rceil\times W-k)=1]\cdot \vx^{(2)}_k, \cdots, \sI[W-(\lceil k/W\rceil\times W-k)=1]\cdot  \vx^{(2)}_k)
    \end{align*}
    by the multiplication between $\ve_{W-(\lceil k/W\rceil\times W-k)}$ and $ \vx^{(2)}_k$ which can be implemented by an MLP with hidden dimension $O(W)$.
    Then, we can get $\vx^{(2)}_{t}$ for $t=\hat g_1,\cdots,\hat g_J,\hat h_1,\cdots,\hat h_K$ in following steps:
    \begin{enumerate}
        \item Get 
        \begin{align*}\left\{
            \begin{aligned}
                &\frac{1}{W}(\vx^{(2)}_{\lceil t/W\rceil\times W-(W-1)}, \cdots, \vx^{(2)}_{\lceil t/W\rceil\times W}) , &i\geq \lceil t/W\rceil\times W\\
                &\frac{1}{i-\lceil t/W\rceil\times(W-1)}(\vx^{(2)}_{\lceil t/W\rceil\times W-(W-1)}, \cdots, \vx^{(2)}_{i},\bm{0},\cdots,\bm{0}), &i< \lceil t/W\rceil\times W
            \end{aligned}\right.
        \end{align*}
        for $t=\hat g_1,\cdots,\hat g_J,\hat h_1,\cdots,\hat h_K$ by using Lemma \ref{lemma:linear transformer aggregation}.
        \item Get
        \begin{align*}\left\{
            \begin{aligned}
            &(\vx^{(2)}_{\lceil t/W\rceil\times W-(W-1)}, \cdots, \vx^{(2)}_{\lceil t/W\rceil\times W}), &i\geq \lceil t/W\rceil\times W\\
            &(\vx^{(2)}_{\lceil t/W\rceil\times W-(W-1)}, \cdots, \vx^{(2)}_{i},\bm{0},\cdots,\bm{0}), &i< \lceil t/W\rceil\times W
            \end{aligned}\right.
        \end{align*}
        for $t=\hat g_1,\cdots,\hat g_J,\hat h_1,\cdots,\hat h_K$ by using Lemma \ref{lemma:MLP_multip} for multiplication first, and use Lemma \ref{lemma:MLP_select} to implement a conditional selection.
        \item Get $\vx^{(2)}_{t}$ for $t=\hat g_1,\cdots,\hat g_J,\hat h_1,\cdots,\hat h_K$ by multiplication between the above vector and $\ve_{W-(\lceil t/W\rceil\times W-t)}$, then sum up them. Then we can get $s_{g^{(\vn)}_1(i)},\cdots,s_{g^{(\vn)}_J(i)};\mathsf{dp}(h^{(\vn)}_1(i)),\cdots,\mathsf{dp}(h^{(\vn)}_K(i))$ by linear projection.
    \end{enumerate}
    If the corresponding field is undefined, we can mark them as some special value. This can be done by selection using MLP (Lemma \ref{lemma:MLP_select}). Assumption \ref{ass:dp_approximation} indicates that we can calculate the function $f$ (defined in (\ref{eq:dp_transition})) using an MLP. The output of this layer is 
    \begin{equation*}
        \vx^{(3)}_k=(\ve_k^{\text{next\_state}}, \ve_k^{\text{dp}}, \ve_k^{\text{next\_dp}}, \vn_k,f^{\text{state}}_k
        , k,1).
    \end{equation*}

    \textbf{Block 4}.
    The fourth block of the autoregressive transformer generates the output based on the flag $f^{\text{state}}_k$.
    We calculate $u(\ve_k^{\text{dp}})$ to get the final result. If $f^{\text{state}}_k=1$, then we select $u(\ve_k^{\text{dp}})$ as the output; otherwise, we prepare and output the DP result for the next state, i.e., $\ve_k^{\text{next\_state}}$ and $\ve_k^{\text{next\_dp}}$.
    This is a conditional selection operation and thus can be implemented by an MLP (Lemma \ref{lemma:MLP_select}).
\end{proof}

\begin{theorem}
    Consider any DP problem satisfying the same condition as in \cref{thm:baseline}. Given any integer $n>0$, let $L$ be the length of the output sequence when the input sequence length is $n$. Then, there is a (log-precision) sparse Transformer with block size $B=\Theta(\sqrt L)$ with a constant depth $M$, a constant number of attention heads $H$, and a hidden dimension $D=\mathrm{O}(\sqrt L)$ that can generate the correct output for all inputs $\vs$ of length $n$.
\end{theorem}

\begin{proof}
    The construction of the first two blocks is the same as the linear transformer, and we will give the construction of the third block and the fourth block for the sparse transformer.

    \textbf{Block 3}. The third block of the Transformer uses $K+J$ heads to perform the following tasks where $K$ and $J$ are defined in (\ref{eq:dp_transition}), refers to the transition in DP:
\begin{itemize}
    \setlength{\itemsep}{0pt}
    \item Get inputs of $s_{g_1(i)},\cdots,s_{g_J(i)}$ where $i$ corresponds to $\ve^{\text{next\_state}}_k$. 
    \item Get DP value of $\mathsf{dp}(h_1(i)),\cdots,\mathsf{dp}(h_K(i))$ for $i$ corresponds to $\ve^{\text{next\_state}}_k$. 
    \item Calculate $\ve^{\text{next\_dp}}_k$, i.e., the DP value of the next state, $\ve^{\text{next\_state}}_k$.
\end{itemize}

    To perform the first two tasks, we need to first calculate the absolute position of the embeddings we want.
    This can be implemented by several MLPs by Assumption \ref{ass:dp_approximation} and problem size $n$.
    Supposing the absolute positions are $\hat g_1,\cdots,\hat g_J,\hat h_1,\cdots,\hat h_K$, we can perform the COPY operation by Lemma~\ref{lemma:sparse_copy}, and we can obtain the input $s_{g_1(i)},\cdots,s_{g_J(i)}$ and the DP value of $\mathsf{dp}(h_1(i)),\cdots,\mathsf{dp}(h_K(i))$ for $i$ corresponds to $\ve^{\text{next\_state}}_k$. With the inputs and DP value, according to Assumption \ref{ass:dp_approximation}, we can calculate the function $f$ (defined in (\ref{eq:dp_transition})) using an MLP. The output of this layer is 
    \begin{equation*}
        \vx^{(3)}_k=(\ve_k^{\text{next\_state}}, \ve_k^{\text{dp}}, \ve_k^{\text{next\_dp}}, \vn_k,f^{\text{state}}_k
        , k,1).
    \end{equation*}

    \textbf{Block 4}.
    The fourth block of the sparse transformer generates the output based on the flag $f^{\text{state}}_k$. We calculate $u(\ve_k^{\text{dp}})$ to get the final result. If $f^{\text{state}}_k=1$, then we select $u(\ve_k^{\text{dp}})$ as the output; otherwise, we prepare and output the DP result for the next state, i.e., $\ve_k^{\text{next\_state}}$ and $\ve_k^{\text{next\_dp}}$.
    This is a conditional selection operation and thus can be implemented by an MLP (Lemma \ref{lemma:MLP_select}).
\end{proof}

\subsection{The proof of \cref{thm:general_lower_bound}}
\label{appendix:lower_dp}

In this subsection, we provide a complete proof of \cref{thm:general_upper_bound}. For clarity, we restate the theorem for the linear transformer and sparse transformer individually and provide their respective full proofs.

\begin{theorem}
    Consider any regular DP problem satisfying the same condition as in \cref{thm:baseline}. Assume that the output sequence length $L$ is proportional to the input sequence length $n$, i.e., $L=\Theta(n)$. Then, given a sufficiently large $n$, for (log-precision) linear Transformer, a model with a constant depth $M$ and a constant number of attention heads $H$ can generate the correct output for all inputs $\vs$ of length $n$ only if the hidden dimension $D=\tilde\Omega(\sqrt L)$.
\end{theorem}

\begin{proof}
By Lemma \ref{lemma: linear transformer as RNN}, we can know the hidden states corresponding to each Linear Transformer layer and the last input token (which is a fixed special token) determines the CoT output. Thus, the size of hidden states should be at least $\Omega(n)=\Omega(L)$. 
By the regularity assumption, we know that different input should corresponds to different hidden states.
This implies that the hidden dimension should be at least $\Omega(\sqrt{L/\log L})=\tilde\Omega(\sqrt{L})$, concluding our proof.
\end{proof}

\begin{theorem}

    Consider any regular DP problem satisfying the same condition as in \cref{thm:baseline}. Assume that the output sequence length $L$ is proportional to the input sequence length $n$, i.e., $L=\Theta(n)$. Then, given a sufficiently large $n$, for (log-precision) sparse Transformer with block size $B=\Theta(\sqrt L)$, a model with a constant depth $M$ and a constant number of attention heads $H$ can generate the correct output for all inputs $\vs$ of length $n$ only if the hidden dimension $D=\tilde\Omega(\sqrt L)$.
\end{theorem}

\begin{proof}
     Assuming $D=o(\frac{\sqrt L}{\ln L})$, we present a proof by contradiction. When the model generates the output sequence, the model attends to at most $\Theta(B)$ tokens in the input sequence, which are the last $c$ tokens of every block and the last $B$ tokens of the input sequence. The overall memory cost of all the hidden embeddings for these tokens is $\Theta(BD\log n)$ bits. According to \cref{eq:sparseattention}, given the same hidden embedding of these tokens in every layer, the model will execute the same computation and produce the identical sequence. Thus, the model can produce a maximum of $e^{\Theta(BD\log n)}=e^{o(L)}$ output sequence types given the hidden dimension $D=o(\frac{\sqrt L}{\ln L})$. However, for an input sequence of length $L$, there are corresponding $e^{\Theta(L)}$ input and output sequence types. According to the pigeonhole principle, this implies the existence of two distinct input sequences that the model generates the same output sequence. According to regularity assumption, there is an input sequence that the model generates an incorrect output sequence. Therefore, we have determined that $D=\Omega(\frac{\sqrt L}{\ln L})=\tilde\Omega(\sqrt L)$.
\end{proof}
\section{Proofs of the theorems in \cref{sec:arith}}

In this section, we will prove the theorems and propositions outlined in Section \ref{sec:arith}. Firstly, we will provide a precise and formal definition of the problems and statements for each theorem and proposition. Subsequently, we will offer a thorough proof for each theorem and proposition.

This paper utilizes the identical formulation of the arithmetic evaluation task and CoT solution presented by \citet{feng2023towards}. The arithmetic evaluation task is denoted as $\mathsf{Arithmetic}(n,p)$ and is defined on the finite field modulo $p$, with the input length not exceeding $n$. The CoT solution can be formally defined as follows: for any arithmetic expression, there must exist a handle, which refers to a pair of adjacent numbers connected by an operator that can be evaluated. In the CoT method, we solve the leftmost variable first in each step and connect the equations of each step using the equal sign.

\subsection{Proofs of \cref{thm:linear_arithmetic}}

In this section, we will give a proof of \cref{thm:linear_arithmetic}.

\begin{theorem}
    For any integer $n$, a log-precision linear Transformer with a constant depth $M$ and a constant number of heads $H$ can generate the correct output for the arithmetic evaluation task for all expressions of length no more than $n$ only if the hidden dimension $D=\tilde\Omega(\sqrt[4] L)$.
\end{theorem}
\begin{proof}
Notice that the generated CoT sequence can be determined by the result of the first step. The different possibilities of results for the first step are at least $\Omega(\exp(n))$ since each $a_1\text{ OP}_1\text{ }a_2\cdots $ are legal where $a_i\in\sF_p$ and $\text{OP}_i\in\{+,-,\times,\div\}$.
On the other hand, the hidden state corresponding to the input sequence determines the output CoT sequence. 
Thus, the size of hidden state should be at least $\Omega(n)=\Omega(\sqrt{L})$. This means that the hidden dimension should be at least $\Omega(\sqrt{\sqrt{L}/\log L})=\tilde{\Omega}(\sqrt[4] L)$, which ends our proof.
\end{proof}

\subsection{Proofs of \cref{thm:sparse_locality}}
In this subsection, we will prove \cref{thm:sparse_locality}, which is a corollary of the theorem from \citet{feng2023towards} such that the sparse transformer can generate the correct output for the DP problem. 

\begin{theorem}[From \citet{feng2023towards}]
    For any DP problem, any integer $n\in\mathbb N$, there exists an autoregressive Transformer with constant depth $L$, hidden dimension $d$ and attention heads $H$ (independent of $n$), such that the answer generated by the Transformer is correct for all input sequences $\vs$ of length no more than $n$. Moreover, all parameter values are bounded by $O(\mathrm{poly}(n))$.
\end{theorem}

\begin{theorem}
\label{thm:standard_dp}
    Consider any $m$-locality DP problem satisfying the same condition as in \cref{thm:baseline}. Given any integer $n>0$, let $L$ be the length of the output sequence when the input sequence length is $n$. Then, there exists a (log-precision) sparse Transformer with block size $B=\Theta(m)$, a constant depth $M$, a constant number of attention heads $H$, and a constant hidden dimension $D$ that can generate the correct output for all inputs $\vs$ of length $n$.
\end{theorem}

\begin{proof}
    Under the assumption of locality, we can treat the sparse transformer as a standard transformer to solve the $m$-locality DP problem. When the standard transformer solves the $m$-locality DP problem, the attention head will only attend to the tokens with the distance at most $m$, and the attention to other tokens is $0$. The sparse transformer adds masks to other tokens, and therefore, is equivalent to the standard transformer for $m$-locality DP problem. According to Theorem~\ref{thm:standard_dp}, the sparse transformer can solve $m$-locality DP problem.
\end{proof}

\subsection{Proofs of \cref{thm:linear_locality}}
In this section, we will give a proof of \cref{thm:linear_locality}, which is a corollary of \cref{thm:general_lower_bound}.
\begin{theorem}
    Consider any $m$-locality regular DP problem satisfying the same condition as in \cref{thm:baseline} and assume that $m=\Theta(n)$ where $n$ is the input sequence length. Then, a log-precision linear Transformer with a constant depth $M$ and a constant number of heads $H$ can generate the correct output for all inputs $\vs$ of length $n$ only if the hidden dimension $D=\tilde\Omega(\sqrt m)$.
\end{theorem}
\begin{proof}
    Same as the arguments used in the proof of \cref{thm:general_upper_bound}, we can get the size of hidden states should be at least $\Omega(n)=\Omega(m)$. 
    The regularity assumption implies that the hidden dimension should be at least $\Omega(\sqrt{m/\log m})=\tilde\Omega(\sqrt{m})$, concluding our proof.
\end{proof}

\section{{Additional Experiments}}
\label{app:experiment}

{The accuracies of standard and efficient Transformers with 5 layers on the ED task are shown in Figure \ref{fig:results_5layers}. The results are similar to the results we obtained when using 3 layers. This evidence further supports our theoretical analysis.}

\begin{figure*}[!h]
    \centering
    \includegraphics[width=1.0\linewidth]{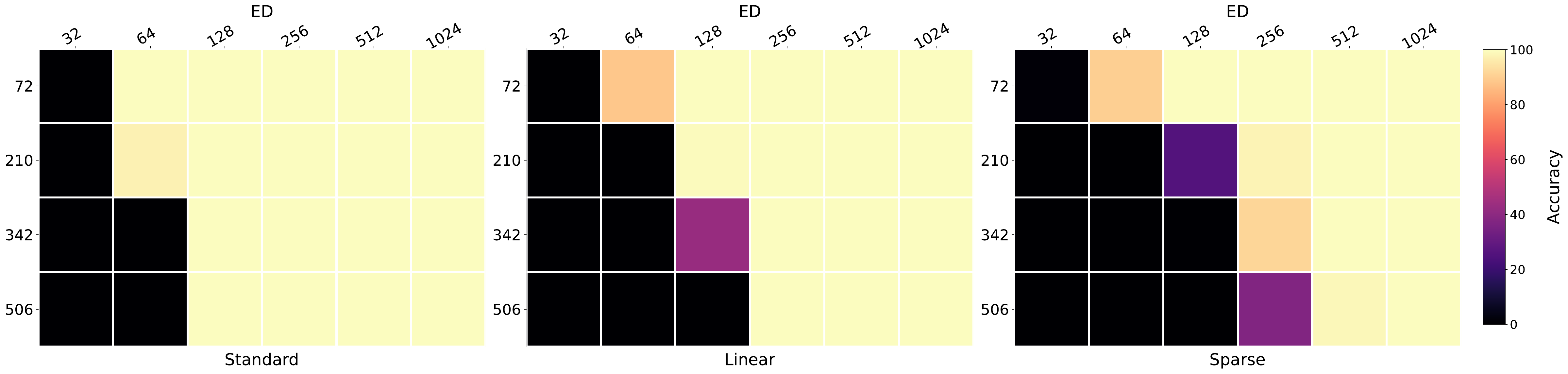}
    \caption{A comparison of accuracies on different model types with ED task. Each subplot corresponds to a model (Standard Transformer, Linear Transformer, Sparse Transformer). Within each subplot, the x-axis represents the embedding dimension, and the y-axis denotes the problem size. The color intensities indicate the accuracy level achieved by the respective models.}
    \label{fig:results_5layers}
    \vspace{-0.4cm}
\end{figure*}


\end{document}